\documentclass[sigconf]{aamas}

\usepackage{algorithm}
\usepackage{algorithmic}
\usepackage{balance} 
\usepackage{fancyhdr}
\usepackage{subcaption}
\usepackage{multirow}

\doi{GYYC3346}

\makeatletter
\gdef\@copyrightpermission{
  \begin{minipage}{0.2\columnwidth}
   \href{https://creativecommons.org/licenses/by/4.0/}{\includegraphics[width=0.90\textwidth]{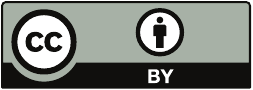}}
  \end{minipage}\hfill
  \begin{minipage}{0.8\columnwidth}
   \href{https://creativecommons.org/licenses/by/4.0/}{This work is licensed under a Creative Commons Attribution International 4.0 License.}
  \end{minipage}
  \vspace{5pt}
}
\makeatother

\setcopyright{ifaamas}
\acmConference[AAMAS '26]{Proc.\@ of the 25th International Conference
on Autonomous Agents and Multiagent Systems (AAMAS 2026)}{May 25 -- 29, 2026}
{Paphos, Cyprus}{C.~Amato, L.~Dennis, V.~Mascardi, J.~Thangarajah (eds.)}
\copyrightyear{2026}
\acmYear{2026}
\acmDOI{}
\acmPrice{}
\acmISBN{}

\title{ME-IGM: Individual-Global-Max in Maximum Entropy Multi-Agent Reinforcement Learning}

\author{Wen-Tse Chen}
\affiliation{
  \institution{Carnegie Mellon University}
  \city{Pittsburgh}
  \country{United States}}
\email{wentsec@andrew.cmu.edu}

\author{Yuxuan Li}
\affiliation{
  \institution{Zhejiang University}
  \city{Hangzhou}
  \country{China}}
\email{yuxuanli04@zju.edu.cn}

\author{Shiyu Huang}
\affiliation{
  \institution{XPeng Inc.}
  \country{China}}
\email{huangsy16@xiaopeng.com}

\author{Jiayu Chen}
\affiliation{
\institution{The University of Hong Kong}  
\institution{INFIFORCE Intelligent Tech. Co., Ltd.}
  \city{Hong Kong}
  \country{China}}
\email{jiayuc@hku.hk}

\author{Jeff Schneider}
\affiliation{
  \institution{Carnegie Mellon University}
  \city{Pittsburgh}
  \country{United States}}
\email{jeff4@andrew.cmu.edu}

\begin{abstract}

Multi-agent credit assignment is a fundamental challenge for cooperative multi-agent reinforcement learning (MARL), where a team of agents learn from shared reward signals.
The Individual-Global-Max (IGM) condition is a widely used principle for multi-agent credit assignment, requiring that the joint action determined by individual Q-functions maximizes the global Q-value. Meanwhile, the principle of maximum entropy has been leveraged to enhance exploration in MARL. However, we identify a critical limitation in existing maximum entropy MARL methods: a misalignment arises between local policies and the joint policy that maximizes the global Q-value, leading to violations of the IGM condition. To address this misalignment, we propose an order-preserving transformation. Building on it, we introduce ME-IGM, a novel maximum entropy MARL algorithm compatible with any credit assignment mechanism that satisfies the IGM condition while enjoying the benefits of maximum entropy exploration. We empirically evaluate two variants of ME-IGM: ME-QMIX and ME-QPLEX, in non-monotonic matrix games, and demonstrate their state-of-the-art performance across 17 scenarios in SMAC-v2 and Overcooked.

\end{abstract}


\newcommand{\BibTeX}{\rm B\kern-.05em{\sc i\kern-.025em b}\kern-.08em\TeX}

\begin{document}


\pagestyle{fancy}
\fancyhead{}


\maketitle


\section{Introduction}
\begin{figure*}[t]
  \centering
  \includegraphics[width=0.45\linewidth]{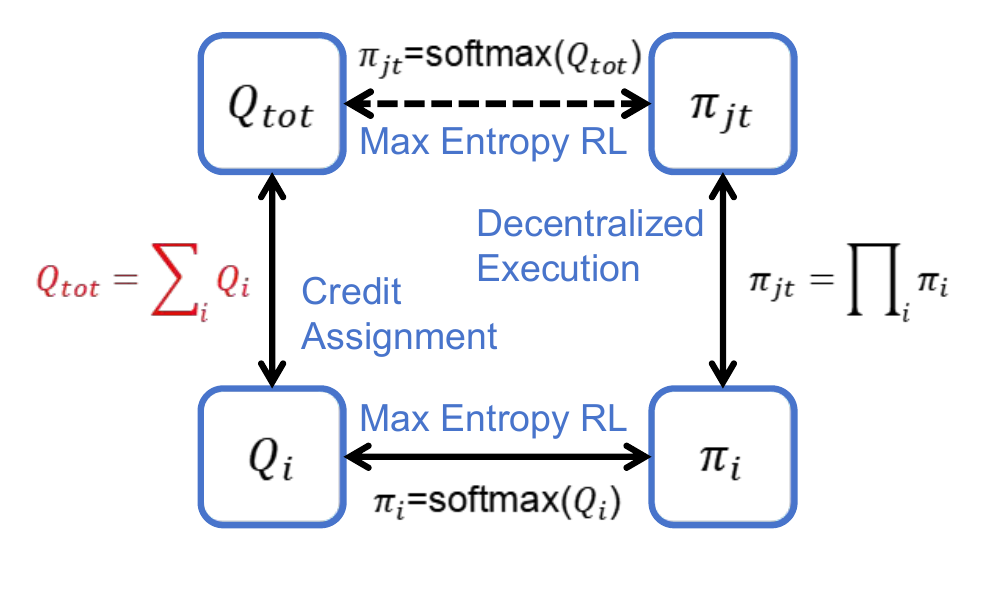}
  \includegraphics[width=0.45\linewidth]{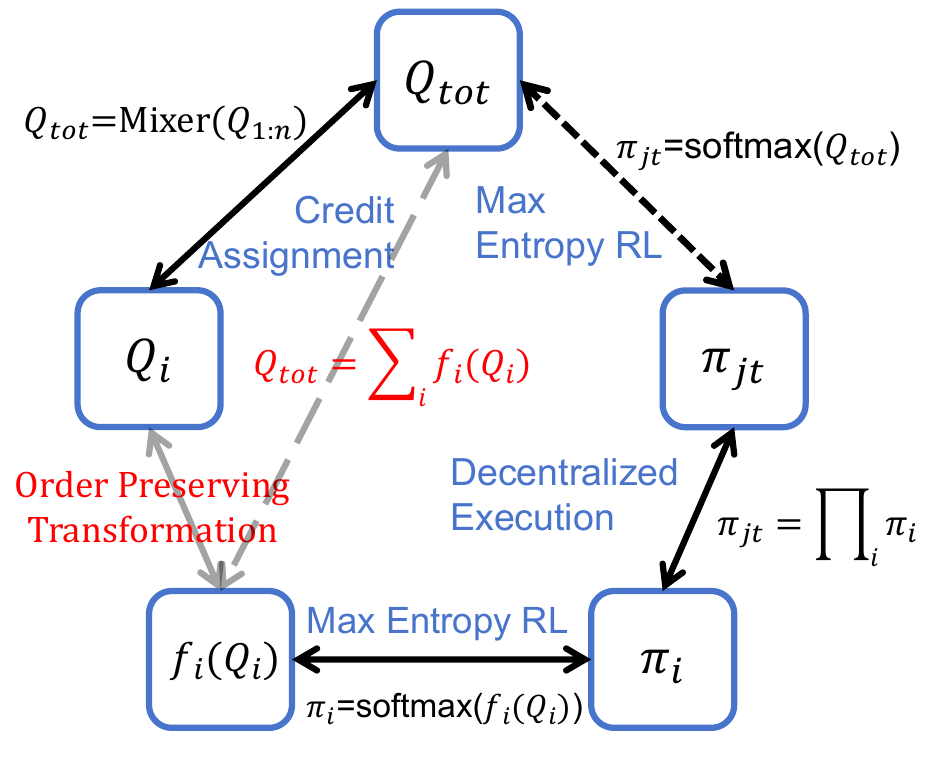}
  \caption{The figure illustrates the improvement of our approach compared to existing maximum entropy MARL methods. The left figure shows a straightforward approach to applying maximum entropy MARL in the CTDE context, where blue texts represent the desired objectives, and black texts indicate corresponding constraints. It reveals that existing maximum entropy MARL methods under the CTDE framework implicitly constrain the global Q-value to be the sum of local Q-values (as in VDN), significantly limiting the expressiveness of the critic network. The right figure depicts the improvements in ME-IGM. ME-IGM first applies any credit assignment mechanism that satisfies the IGM condition, such as QMIX and QPLEX, to obtain \(Q_i\) for each agent $i$. Given the meaningful order of local Q-values (for different actions), an order-preserving transformation $f_i$ converts \(Q_i\) to \(f_i(Q_i)\) as the policy logits. This transformation is optimized using a loss function that minimizes the expected difference between \(\sum_i f_i(Q_i)\) and \(Q_{tot}\), which guarantees monotonic policy improvement in maximum entropy MARL.}
  \Description{This image demonstrates improvement of our approach compared to existing maximum entropy MARL methods}
  \label{fig:teaser}
\end{figure*}


Collaborative multi-agent tasks, where a team of agents works together to complete a task and receives joint rewards, are inherently challenging for reinforcement learning (RL) agents. This difficulty arises due to the absence of individual reward signals for each agent and the significantly higher exploration requirements compared to single-agent scenarios. The challenge is further exacerbated by the exponential growth of the joint action space as the number of agents increases. In this work, we propose leveraging maximum entropy RL to promote exploration and adopting the individual-global-max (IGM) credit assignment mechanism~\citep{rashid2020monotonic} to effectively distribute joint rewards among the agents.


Maximum entropy RL~\cite{levine2018reinforcement}, recognized for promoting exploration~\cite{haarnoja2018soft}, smoothing the objective landscape~\cite{ahmed2019understanding}, and improving robustness~\cite{eysenbach2021maximum}, has been extensively applied in single-agent RL. However, as illustrated in Figure~\ref{fig:teaser}, naively extending maximum entropy RL to multi-agent settings under the centralized training and decentralized execution (CTDE) paradigm leads to uniform credit assignment, which limits the ability of multi-agent reinforcement learning (MARL) agents to learn from joint rewards. 

On the other hand, numerous studies have explored the development of effective credit assignment mechanisms. Among them, the IGM condition has demonstrated strong empirical performance by aligning the maximum local Q-values with the maximum global Q-value. Previous works have primarily applied the IGM condition in deterministic policy settings, where policies are derived using the argmax of local Q-values. 
However, in maximum entropy MARL methods, stochastic local policies are learned, and misalignment can occur between local policies and the maximum global Q-values. \textbf{That is, selecting joint actions based on the highest logits of local policies does not necessarily lead to maximizing the global Q-value.} 
This limitation highlights a critical challenge in achieving effective credit assignment within the maximum entropy MARL framework. This challenge, as we will demonstrate, is not unique to IGM-based methods but also represents a broader obstacle for actor-critic-based maximum entropy MARL algorithms, such as FOP~\citep{zhang2021fop}, hindering their ability to achieve optimal coordination.

We propose ME-IGM, the first CTDE maximum entropy MARL method that is compatible with 
the IGM condition. In particular, to address the misalignment between stochastic local policies and the maximum global Q-values, we introduce an order-preserving transformation (OPT) that maps local Q-values to policy logits while preserving their relative order. This ensures that selecting actions with the highest logits collectively leads to a joint action that maximizes the global Q-value. Additionally, the OPT is trained with a theoretically grounded objective, guaranteeing monotonic policy improvement in CTDE MARL.


Our main contributions are summarized as follows:
\begin{itemize}
\item We are the first to identify and empirically demonstrate a critical issue of existing maximum entropy MARL methods: the misalignment between stochastic local policies and the maximum global Q-values.
\item We introduce ME-IGM, the first CTDE maximum entropy MARL method that fully complies with the IGM condition. The core of ME-IGM is an order-preserving transformation (OPT), which maps local Q-values to policy logits while preserving their relative order. This ensures that selecting actions with the highest logits collectively leads to a joint action that maximizes the global Q-value. OPT can be seamlessly integrated into mainstream value decomposition MARL methods for a maximum entropy extension to enhance exploration. Additionally, we propose a theoretically grounded and straightforward objective for training these transformation operators.
\item We confirm the effectiveness of our algorithm through extensive evaluation on matrix games, SMAC-v2~\citep{ellis2022smacv2},  and Overcooked~\cite{carroll2019utility}, where our method attains state-of-the-art results.~\footnote{We have published our code at \url{https://github.com/WentseChen/Soft-QMIX}}
\end{itemize}

\section{Background}
\label{sec:preliminary}

\subsection{Multi-Agent Reinforcement Learning}
Cooperative MARL can be formalized as a Decentralized Partially Observable Markov Decision Process (Dec-POMDP)~\cite{oliehoek2016concise}. Formally, a Dec-POMDP is represented as a tuple \((\mathcal{A},S,U,T,r,O,G,\gamma)\), where \(\mathcal{A}\equiv\{1,...,n\}\) is the set of $n$ agents. $S$, $U$, $O$, and \(\gamma\) are state space, action space, observation space, and discount factor, respectively. 

During each discrete time step \(t\), each agent $i\in \mathcal{A}$ chooses an action $u^i\in U$, resulting in a collective joint action $\mathbf{u}\in \mathbf{U}\equiv U^n$. 
The function \(r(s,\mathbf{u})\) defines the immediate reward for all agents when the collective action \(\mathbf{u}\) is taken in the state \(s\).
$\mathcal{P}(s,\mathbf{u},s'):S\times\mathbf{U}\times S\rightarrow [0,\infty)$ is the state-transition function, which defines the probability density of the succeeding state $s'$ after taking action $\mathbf{u}$ in state $s$. In a Dec-POMDP, each agent receives only partial observations $o\in O$ according to the observation function $G(s,i): S\times\mathcal{A}\rightarrow O$. For simplicity, if a function depends on both $o^i$ and $s$, we will disregard $o^i$. Each agent uses a policy $\pi_i(u^i|o^i)$ to produce its action $u^i$. Note that the policy should be conditioned on the observation history $o^i_{1:t}$. For simplicity, we refer to $o^i_{1:t}$ as $o^i_t$ or $o^i$ in the whole paper. The joint policy is denoted as \(\pi_{jt}\). \(\rho_\pi(s_t, \mathbf{u_t})\) is used to represent the state-action marginals of the trajectory distribution induced by a policy \(\pi\).

\subsection{Individual-Global-Max Condition}
Individual-Global-Max (IGM)~\cite{rashid2020monotonic}  is a commonly used credit assignment method in value-decomposition-based cooperative MARL and is defined as follows:

\begin{definition}
\textbf{Individual-Global-Max (IGM)}
\label{def:IGM}
For joint q-function \(Q_{tot}\), if there exist individual q-functions \(\left[Q_{i}\right]_{i=1}^{n}\) such that the following holds:
\begin{equation}
\label{eq:action_max}
\underset{\mathbf{u_t}}{\mathrm{argmax}} \ Q_{tot}(\mathbf{u_t}, s_t) = 
\begin{pmatrix}
\underset{u_t^1}{\mathrm{argmax}} \ Q_1(u^1_t|o^1_t) \\
\vdots \\
\underset{u_t^n}{\mathrm{argmax}} \ Q_n(u^n_t|o^n_t)
\end{pmatrix}.
\end{equation}
Then, we say that $[Q_i]_{i=1}^{n}$ satisfy \(\textbf{IGM}\) for $Q_{tot}$ under $s_t$.
\end{definition}



\subsection{Related Works} \label{rws}

{\bf Multi-Agent RL.} In the CTDE framework, MARL primarily involves two common algorithmic approaches: policy gradient methods~\cite{lowe2017multi,yu2022surprising,lin2023tizero,foerster2018counterfactual,kuba2021trust,zhong2023heterogeneousagent} and value function decomposition methods~\cite{sunehag2017value,rashid2020monotonic,wang2020qplex}. Policy gradient methods first learn a centralized critic network, and then distill local policies using losses like the KL divergence. On the other hand, Value function decomposition addresses the credit assignment problem by decomposing the global value function into multiple local value functions. VDN~\cite{sunehag2017value} assumes that the global Q-function is the sum of local Q-functions.
QMIX~\cite{rashid2020monotonic,peng2021facmac} permits the mixer function to be any function with non-negative weights. QTRAN~\cite{son2019qtran} optimizes an additional pair of inequality constraints to construct a loss function.
QPLEX~\cite{wang2020qplex} use the dueling architecture to decompose the Q-function, achieving the same expressive power as QTRAN while being easier to optimize. The expressive power of the critic networks in these algorithms increases sequentially, but experimental results show that QMIX has better performance~\cite{hu2021rethinking}. Therefore, in this work, QMIX is used as the mechanism for credit assignment.

{\bf Maximum Entropy MARL.} Maximum entropy RL~\cite{levine2018reinforcement} is demonstrated to have advantages such as encouraging exploration~\cite{haarnoja2018soft} and increasing robustness~\cite{eysenbach2021maximum} in single-agent RL scenarios. In maximum entropy MARL, most works utilize an actor-critic (AC) architecture and typically aim to minimize the KL divergence between the local actor and the centralized critic~\cite{he2022multiagent, guo2022learning, pu2021decomposed, zhang2021fop}. 

\textbf{Comparison with Actor-Critic Approaches.} 
Most existing maximum entropy MARL methods (e.g., mSAC~\cite{pu2021decomposed}, FOP~\cite{zhang2021fop}, and HASAC~\cite{liu2023maximum}) rely on the Actor-Critic framework. Our value-based ME-IGM differs from them in three aspects. 
First, \textbf{Order Preservation}: AC-based methods may suffer from a misalignment where the order of actions preferred by the local actor does not match the local Q-function due to approximation errors. In contrast, our method uses Order-Preserving Transformations (OPT) to strictly enforce this alignment, ensuring the IGM condition is met. 
Second, \textbf{Objective Function}: AC methods typically minimize KL divergence to distill policies. Our method minimizes the MSE loss between logits, which has been shown to be superior in distillation contexts~\cite{kim2021comparing}. 
Third, \textbf{Specific Baselines}: FOP~\cite{zhang2021fop} employs a credit assignment mechanism similar to QPLEX but relies on the Individual-Global-Optimal (IGO) assumption. While FOP argues that locally optimizing policies reduces the joint policy gap, it does not explicitly guarantee alignment. HASAC~\cite{liu2023maximum} avoids the IGO assumption through sequential optimization but ignores the misalignment problem between stochastic policies and Q-values. Our work identifies this misalignment as a critical bottleneck and proposes OPT to resolve it, thereby achieving state-of-the-art performance.

\section{Misalignment Between Local Policies and the Maximum Global Q-Value} \label{misalignment}



\textbf{The Misalignment Problem of Maximum Entropy MARL:} 
While the IGM condition defines the relationship between local Q-values and the global Q-value, what is often more important in practice is the alignment between local policies and the global Q-value. In other words, joint actions that maximize the global Q-value can be obtained by querying local policies. 
Previous work in maximum entropy MARL~\cite{zhang2021fop, guo2022learning} often decompose the global Q-value into individual local Q-values to satisfy the IGM condition. These local Q-values are then used to guide the training of local stochastic policies. 
Compared with value decomposition methods such as QMIX, they train stochastic policy networks in addition to the Q-functions, which can potentially improve multi-agent exploration in complex environments.
\textbf{However, while the value decomposition ensures alignment between global and local Q-values via the IGM condition, there is typically no explicit constraints imposed between the local policies and the local Q-values during the policy training phase. This lack of coupling often leads to misalignment during testing, where the local policies fail to consistently select joint actions that maximize the global Q-value, defeating the purpose of enforcing the IGM condition.} Consequently, the learned policies may struggle to achieve optimal global coordination during execution.


\begin{figure}[t]
\centering
\includegraphics[width=0.35\textwidth]{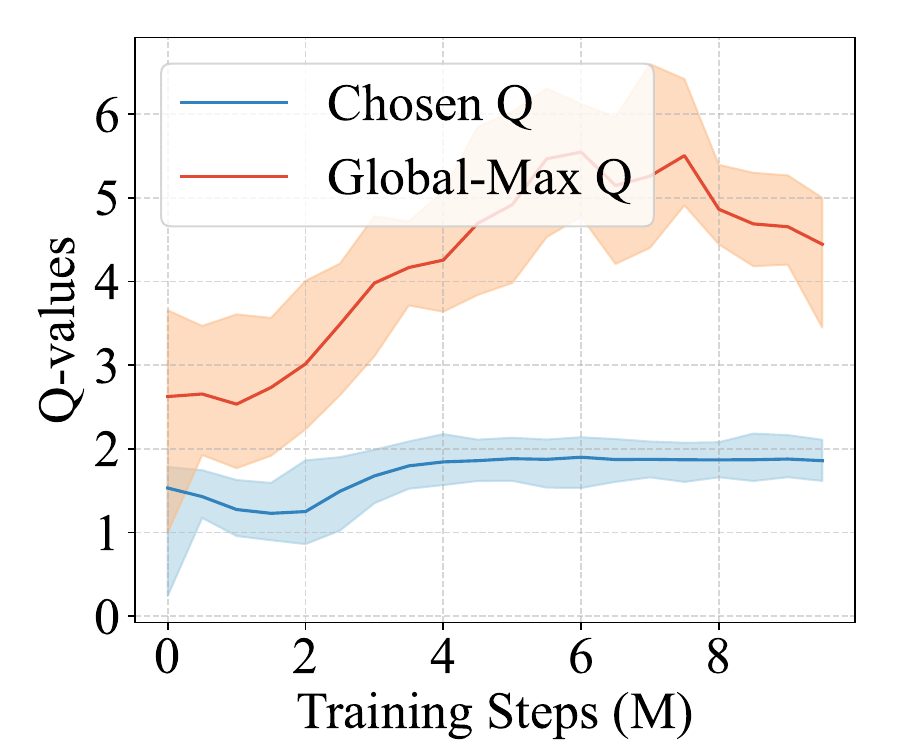}
    \caption{Illustration of misalignment between local policies and the maximum global Q-value. When naively combining the IGM condition with maximum entropy MARL, local policies often select suboptimal joint actions, leading to a lower Q-value, depicted by the blue curve. 
    In contrast, the optimal joint action should achieve the global-max Q-value, shown as the orange curve. 
    }
    \Description{Illustration of misalignment}
    \label{fig:Q_delta}
\end{figure}

\noindent\textbf{Empirical Demonstration:} To analyze the performance gap caused by the misalignment between local policies and the maximum global Q-value, we evaluate the difference between the maximum global Q-value and the global Q-value selected by local policies in the SMAC-v2 \textit{zerg 5vs5} scenario, using a state-of-the-art Maximum Entropy MARL algorithm~\citep{zhang2021fop}. Figure~\ref{fig:Q_delta} shows the results, averaged over five random seeds with standard deviations included. 
We can observe that there is a significant Q-value gap, emphasizing the potential for improvement by developing algorithms that enable the selection of joint actions with the highest global Q-value through local stochastic policies.


\section{Order-Preserving Transformation}

Here, we propose a novel order-preserving transformation (OPT) to solve the aforementioned misalignment problem between local policies and the maximum global Q-value. 

\subsection{Formulation of OPT}

OPT is used to map local Q-values to the logits of local policies while preserving their relative order. 
Formally, the OPT is defined as: 
\begin{equation}
\label{eq:order_preserving_transformation}
\begin{aligned}
\text{OPT}(x,s) = W^2\sigma(W^1x+b^1)+b^2
\end{aligned}
\end{equation}
Here, \(W^1_{_{(d_1\times 1)}}\), \(W_{_{(1\times d_1)}}^2\), \(b^1_{_{(d_1\times 1)}}\), and \(b^2_{_{(1\times 1)}}\) are generated by a hyper-network\footnote{For details of the hyper-network, please refer to the Appendix.} which takes the global state $s$ as input. 
Consider \( x \) as a local Q-value of dimension \( (1 \times d_2) \), then \( W^1x + b^1 \) is a set of hidden variables of size \( (d_1 \times d_2) \).
Notably, when all entries in \( W^1 \) are non-negative, the relative order of elements in \( x \) is preserved in each of the $d_1$ hidden variables.  
Next, a nonlinear activation function \( \sigma \), such as ELU~\cite{clevert2015fast}, is applied, followed by another order-preserving linear transformation using \( W^2 \) and \( b^2 \). Consequently, the final output, \( \text{OPT}(x, s) \), maintains both the dimension and the relative order of the entries in the input \( x \). 

$\text{OPT}(x,s)$ is used as the logits of the local policy corresponding to the local Q-value $x$. Through aforementioned order-preserving design, each agent can select its action corresponding to the highest local policy logit, which is also the maximum point of its local Q-function. The local Q functions can be learned with any value decomposition methods satisfying the IGM condition, to ensure that the selected local actions collectively compose the optimal joint action that maximize the global Q-function.
Thus, the proposed OPT addresses the aforementioned misalignment issue commonly found in maximum entropy MARL.

\subsection{Training Scheme of OPT}
\label{sec:train_opt}

\begin{figure}[h!]
\centering
\includegraphics[width=0.35\textwidth]{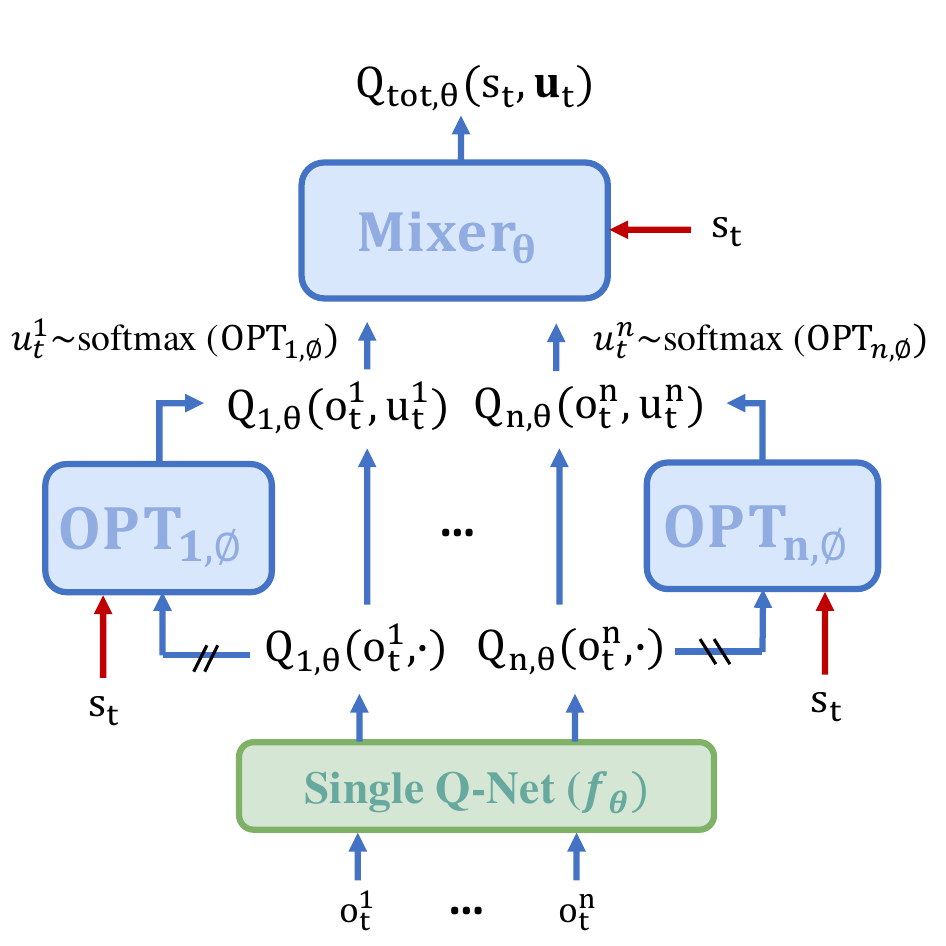}
\caption{The overall pipeline of ME-IGM.}
\Description{The overall pipeline of ME-IGM.}
    \label{fig:network}
\end{figure}

Here, we first formulate the ideal objective of Maximum Entropy MARL and its practical counterpart that admits an analytical-form solution. We then analyze the gap between these two objective functions, providing a theoretically sound objective design for Maximum Entropy MARL based on OPTs.

Maximum Entropy MARL trains the individual and global Q-functions as in value decomposition MARL methods, while training a policy network in a Maximum Entropy RL framework based on the global Q-values. Formally, the joint policy is trained as follows:
\begin{equation}
\label{eq:kl_divergence}
\underset{\pi'_{jt}\in\Pi}{\mathrm{min}}
\ D_{KL} \left( \pi'_{jt}(\cdot | s_t) \middle\| \frac{\exp(Q_{tot}^{\pi_{\text{old}}}(s_t, \cdot)/\alpha)}{Z^{\pi_{\text{old}}}(s_t)} \right),
\end{equation}
where \(Z^{\pi_{\text{old}}}(s_t) = \sum_\mathbf{u}\exp(Q_{tot}^{\pi_{\text{old}}}(s_t, \mathbf{u})/\alpha)
\) and $\alpha$ is the temperature in the softmax distribution. 

To simplify and decentralize the policy improvement process, we propose replacing \( Q_{tot}^{\pi_{\text{old}}}(s_t, \cdot) \) with \( \sum_i \text{OPT}_i(Q_i^{\pi_{\text{old}}}(o^i_t, \cdot), s_t) \), which leads to:
\begin{equation}
\underset{\pi'_{jt}\in\Pi}{\mathrm{min}}
\ D_{KL} \left( \pi'_{jt}(\cdot | s_t) \middle\| \frac{\exp(\sum_i\frac{\text{OPT}_i(Q_i^{\pi_{\text{old}}}(o^i_t, \cdot), s_t)}{\alpha})}{Z'^{\pi_{\text{old}}}(s_t)} \right),
\label{eq:real_policy_improvement}
\end{equation}
where \(Z'^{\pi_{\text{old}}}(s_t)\) is the normalization factor.
The joint policy $\pi'_{jt}$ is implemented as a set of independent local policies for decentralized execution, and we have the analytical solution for the local policy of agent $i$ as:
\begin{equation}
\pi_i^{\text{new}}(a|o^i_t) \propto  \exp(\text{OPT}_i(Q_i^{\pi_{\text{old}}}(o^i_t, a), s_t) / \alpha).
\label{eq:policy_format}
\end{equation}
We can see that the policy is defined based on the local Q-value and utilizes the OPT to ensure that the relative order of policy logits aligns with the local Q-values.

\textbf{There is a gap between $\sum_i\text{OPT}_i(Q_i^{\pi_{\text{old}}}(o^i_t, \cdot), s_t)$ and $Q_{tot}^{\pi_{\text{old}}}(s_t, \cdot)$.} Also, we find that the policy improvement using Equation (\ref{eq:real_policy_improvement}) no longer guarantees a monotonic increase in the global Q-value, unlike the approach based on Equation (\ref{eq:kl_divergence}). Specifically, we have the following theorem:
\begin{theorem}
\label{thm:policy_improvement_gap}
Denote the policy before the policy improvement as \(\pi_{\text{old}}\) and the policy achieving the optimality of the improvement step as \(\pi_{\text{new}}\).
Updating the policy with Equation (\ref{eq:kl_divergence}) ensures monotonic improvement in the global Q-value. That is, $Q^{\pi_{\text{old}}}_{tot}(s_t, \mathbf{u_t}) \leq Q_{tot}^{\pi_{\text{new}}}(s_t, \mathbf{u_t}), \forall s_t,\mathbf{u_t}.$

\begin{algorithm}[ht]
\caption{ME-IGM}
\label{alg:soft-QMix}
\begin{algorithmic}[1]
\STATE \textbf{Input:} $\theta, \phi$
\STATE \textbf{Initiate:} ${\theta^-} \gets \theta, D \gets \emptyset$ 
\FOR{each iteration}
    \FOR{each environment step}
        \STATE get $\pi_i(u^i_t|o^i_t,s_t)$ via Eq.~\ref{eq:policyi}
        \STATE $u^i_t \sim \pi_{i}(u^i_t|o^i_t,s_t), \forall i\in\{1,..,n\}$ 
        \STATE $s_{t+1} \sim \mathcal{P}(s_t, \mathbf{u}_t, \cdot)$ 
        \STATE $D \gets D \cup \{(s_t, \mathbf{u}_t, r(s_t, \mathbf{u}_t), s_{t+1})\}$ 
    \ENDFOR
    \FOR{each gradient step}
        \STATE calculate $\mathcal{T}^{\pi_{jt}}_\lambda Q_{tot}$  via Eq.
        \STATE $\theta \gets \theta - \lambda_\theta \nabla_{\theta} L(\theta)$
        \STATE $\phi \gets \phi - \lambda_\phi \nabla_\phi L(\phi)$
        \STATE $\omega \gets \omega - \lambda_\omega \nabla_\omega L(\omega)$
        \STATE ${\theta^-} \gets \tau\theta + (1 - \tau){\theta^-}$
    \ENDFOR
\ENDFOR
\end{algorithmic}
\end{algorithm}

In contrast, updating the policy with Equation (\ref{eq:real_policy_improvement}) only ensures that $Q^{\pi_{\text{old}}}_{tot}(s_t, \mathbf{u_t}) -\epsilon
\leq Q_{tot}^{\pi_{\text{new}}}(s_t, \mathbf{u_t}), \forall s_t,\mathbf{u_t}$.
The gap $\epsilon$ satisfies:
\begin{equation}
\begin{aligned} \label{equ:4}
\epsilon \leq \frac{\gamma}{1-\gamma}\mathbb{E}_{s_{t+1}\sim \mathcal{P}}\big[C\sqrt{2 D_{KL}\big(\pi_{\text{old}}(\cdot|s_{t+1})|\pi_{\text{new}}(\cdot|s_{t+1})\big)}\big]
\end{aligned}
\end{equation}
where $\gamma$ is the discount factor, $\mathcal{P}(s_t, \mathbf{u_t}, \cdot)$ is the transition function, $C  = \underset{\mathbf{u}_{t+1}}{\mathrm{max}}\ \Delta_Q(s_{t+1}, \mathbf{u}_{t+1})$, and $ \Delta_Q(s_{t+1}, \mathbf{u}_{t+1})\triangleq |Q_{tot}^{\pi_{\text{old}}}(s_{t+1}, \mathbf{u}_{t+1}) - $ $\sum_i \text{OPT}_i(Q_i^{\pi_{\text{old}}}(o^i_{t+1}, u^i_{t+1}), s_{t+1})|$.
\end{theorem}

The proof is provided in Appendix~\ref{ap:proof_epsilon}, where we also analyze how this error affects the general convergence properties of policy improvement.
This motivates us to minimize the factor \(C\) and adopt the following objective to train the OPTs.
\begin{equation}
\begin{aligned} 
\label{equ:delta_q_objetive}
\min_{{\text{OPT}_{1:n}}} \mathbb{E}_{s, \mathbf{u}\sim \rho_{\pi_{jt}^{new}}}[\Delta_Q(s, \mathbf{u})^2],
\end{aligned}
\end{equation}
where $\pi_{jt}^{new} = (\pi_{1}^{new}, \cdots, \pi_{n}^{new})$ and $\pi_{i}^{new}$ is defined in Eq. (\ref{eq:policy_format}).

\begin{table}[t]
    \caption{(a) presents the payoff matrix for a one-step matrix game involving two agents: the row player and the column player. Each agent has three possible actions: \(\{\text{A, B, C}\}\). The rewards for each joint action are given in a \(3 \times 3\) matrix. This payoff structure is non-monotonic, as the optimal action for one player depends on the action chosen by its teammate. (b)–(d) illustrate the results of ME-QMIX, FOP, and QMIX, respectively. The first row and first column display the policies adopted by the row and column players after \(10k\) training steps. The \(3 \times 3\) matrix in the bottom right corner represents the global Q-function \( Q_{\text{tot}} \) learned by each algorithm. We can see that ME-QMIX is the only method that assigns the highest probability to selecting the optimal joint action, highlighting its effectiveness in overcoming the misalignment issue and achieving optimal coordination. $\epsilon$ in (d) denotes the exploration rate used in $\epsilon$-greedy exploration.}
    \begin{minipage}{.48\linewidth}
      \centering
      \subcaption{Payoff Matrix}
        \begin{tabular}{|c||c|c|c|}
        \hline
        & A & B & C\\ \hline \hline
        A & \textbf{8}& -12 & -12 \\ \hline
        B & -12 & 0 & 0 \\ \hline
        C & -12 & 0 & 0 \\ \hline
        \end{tabular}
        \label{table:game1_payoff}
    \end{minipage}
    \begin{minipage}{.48\linewidth}
      \centering
      \subcaption{Result of ME-QMIX (Ours)}
        \begin{tabular}{|c||c|c|c|}
        \hline
        & 1. & 0. & 0.\\ \hline \hline 
        1. & \textbf{8}& -12 & -14 \\ \hline
        0. & -8 & -28 & -30 \\ \hline
        0. & -10 & -30 & -32 \\ \hline
        \end{tabular}
        \label{table:game1_softQMIX}
    \end{minipage}\\
    \begin{minipage}{.48\linewidth}
      \centering
        \subcaption{Result of FOP}
        \begin{tabular}{|c||c|c|c|}
        \hline
        & .21 & .39 & .40\\ \hline \hline
        .21 & \textbf{8}& -12 & -12 \\ \hline
        .40 & -12 & 0 & 0 \\ \hline
        .39 & -12 & 0 & 0 \\ \hline
        \end{tabular}
        \label{table:game1_FOP}
    \end{minipage}
    \begin{minipage}{.48\linewidth}
      \centering
        \subcaption{Result of QMIX}
        \begin{tabular}{|c||c|c|c|}
        \hline
        & \(\epsilon\) & 1-\(2\epsilon\) & \(\epsilon\)\\ \hline \hline
        \(\epsilon\) & -8.5 & -8.5 & -8.5 \\ \hline
        1-\(2\epsilon\) & -8.5 & \textbf{0.2} & 0.2 \\ \hline
        \(\epsilon\) & -8.5 & 0.2 & 0.2 \\ \hline
        \end{tabular}
        \label{table:game1_QMIX}
    \end{minipage}
\label{tb:matrix_game1}
\end{table}

\section{The Overall Framework: ME-IGM}

Figure~\ref{fig:network} illustrates the framework of ME-IGM, which consists of three types of networks. Agents 1 to \(n\) share a network \(f_\theta\) for predicting local Q-values. The Mixer network, which aggregates local Q-values and manages the credit assignment among agents, is parameterized through its associated hyper-network as \(\text{Mixer}_\theta\)\footnote{Our algorithm is compatible with any Mixer network design that satisfies the IGM condition. For instance, during evaluation, we utilize the Mixer networks of QMIX and QPLEX, leading to two variants of our algorithm: ME-QMIX and ME-QPLEX.}. \([\text{OPT}_{i, \phi}]_{i=1}^n\) are \(n\) order-preserving transformations, and  $\text{OPT}_{i, \phi}(f_\theta(o_t^i), s_t)$ defines a local stochastic policy for agent $i$.
The update rules for these networks are described as follows.

ME-IGM optimizes two loss functions to update \(\theta\) and \(\phi\) separately. \(\text{Mixer}_\theta\) and \(f_\theta\) are trained by minimizing \(L(\theta)\):
\begin{equation}
\label{eq:LQ}
\mathbb{E}_{\tau_t} \left[\frac{1}{2}\left(Q_{tot_\theta}(s_t, \mathbf{u}_t) -\widehat{\mathcal{T}}^{\pi_{jt}} Q_{tot}(s_t, \mathbf{u}_t)\right)^2\right]
\end{equation}
where the variables are defined as:
\begin{equation*}
    \tau_t = \{o_{t:t+1}^{1:n},\; u_{t:t+1}^{1:n},\; s_{t:t+1},\; r_{t}\}
\end{equation*}
\begin{equation*}
    Q_{tot_\theta}(\mathbf{u},s)
        = \text{Mixer}_\theta\big(f_\theta(o^1,u^1), \ldots, f_\theta(o^n,u^n), s\big)
\end{equation*}
\text{and the estimated Q-target is:}
\begin{equation*}
    \widehat{\mathcal{T}}^{\pi_{jt}}Q_{tot}(s_t, \mathbf{u}_t)
        = r_t + \gamma Q_{tot,\theta^-}(s_{t+1}, \mathbf{u}_{t+1})
\end{equation*}
As a common practice, we utilize a target Q network $Q_{tot,\theta^-}$ to stabilize training, with \(\theta^-\) being the exponentially weighted moving average of \(\theta\). Notably, local actions \( u_{t:t+1}^{1:n} \) are sampled from local stochastic policies \( \pi_i \) (as defined in Equation (\ref{eq:policyi})) during roll-out, rather than selecting the argmax of local Q-functions. This offers two advantages: (1) the softmax policy function enhances exploration, and (2) incorporating state \( s \) in the sampling process allows for better utilization of global information during centralized training.

Further, to train \([\text{OPT}_{i, \phi}]_{i=1}^n\), we adopt a sample-based version of Equation (\ref{equ:delta_q_objetive}), denoted as $L(\phi)$:
\begin{equation}
\label{eq:Lpi}
\mathbb{E}_{\tau_t}\left[\left(\sum_{i=1}^n \text{OPT}_{i,\phi}(f_{\theta}(o^i_t,u^i_t), s_t)- Q_{tot,\theta}(s_t, \mathbf{u}_t)\right)^2\right]
\end{equation}

OPTs define the local policies for each agent as follows:
\begin{equation}
\label{eq:policyi}
\pi_i(u^i|o^i,s) = \text{softmax}\left(\text{OPT}_{i,\phi}\left(f_\theta(o^i),s\right)/\alpha_\omega\right)
\end{equation}
$\alpha_\omega$ is the temperature and, similar to~\citet{haarnoja2018soft}, can be updated by minimizing the following loss function:
\begin{equation}
L(\omega) = \mathbb{E}_{\tau_t} 
\left[ -\alpha_\omega \log \pi_{jt}(\mathbf{u}_t | s_t) - \alpha_\omega \bar{\mathcal{H}} \right],
\end{equation}
where $\bar{\mathcal{H}}$ is a predefined target entropy. Note that the local policies (i.e., $\pi_i$) are only used to collect training roll-outs $\tau_t$. During execution, we would select the argmax of \( f_\theta \), which does not require any centralized information (e.g., $s$), since OPTs preserve the order of the input vectors. The pseudo code of ME-IGM algorithm is presented in Algorithm~\ref{alg:soft-QMix}.

\section{Experiments}

In this section, we first evaluate our method on a classic matrix game, demonstrating its capability to learn a globally optimal policy despite the non-monotonic nature of the overall payoff matrix. Next, we compare our algorithm against a set of baselines on the SMAC-v2 and Overcooked benchmark, highlighting its state-of-the-art performance in cooperative MARL. Moreover, we conduct ablation studies to evaluate the benefits of entropy-driven exploration, assess the necessity of incorporating the order-preserving transformation, and analyze the impact of key hyperparameters.

\begin{figure}[t]
\centering
\includegraphics[width=0.4\textwidth]{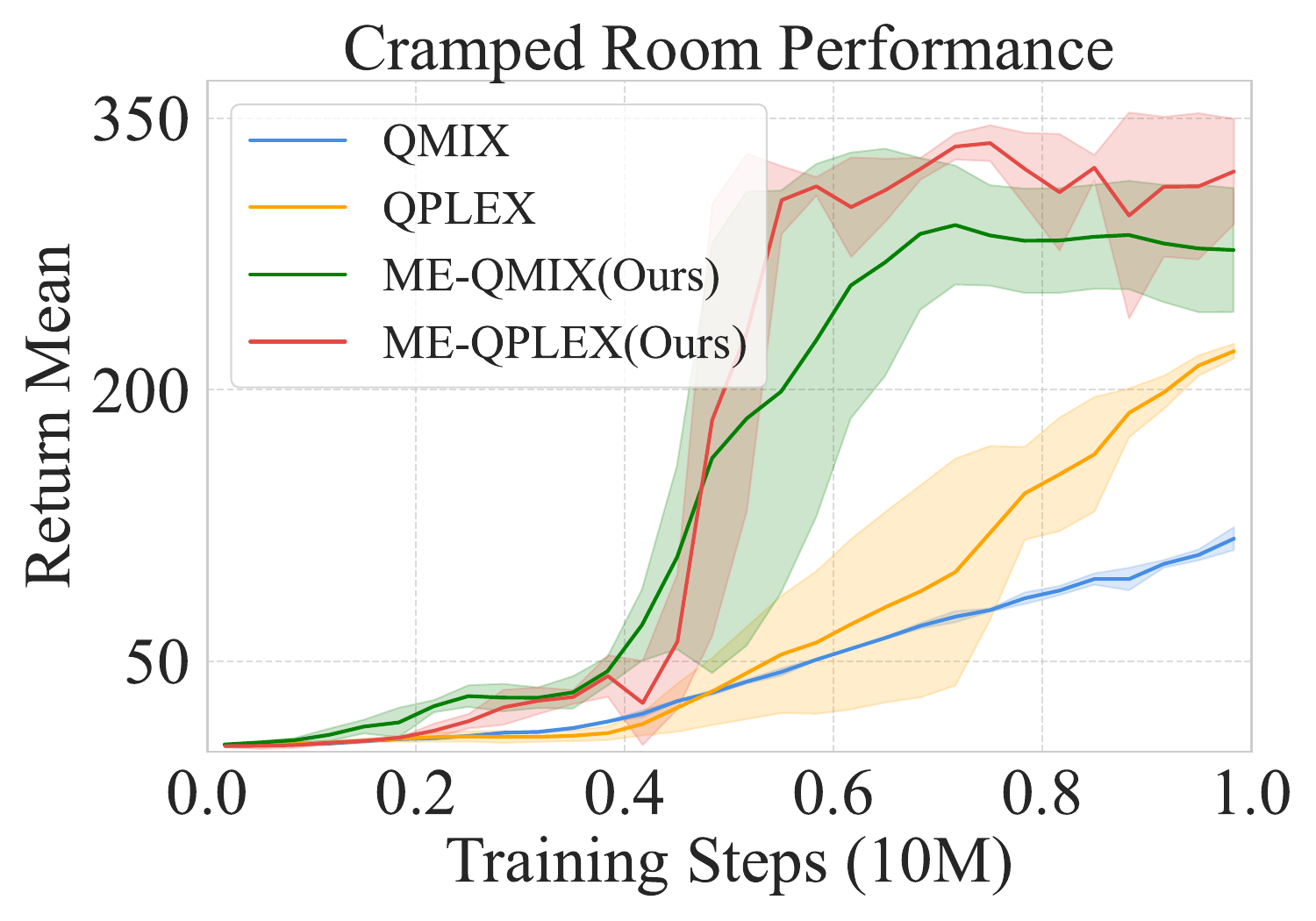}
\includegraphics[width=0.4\textwidth]{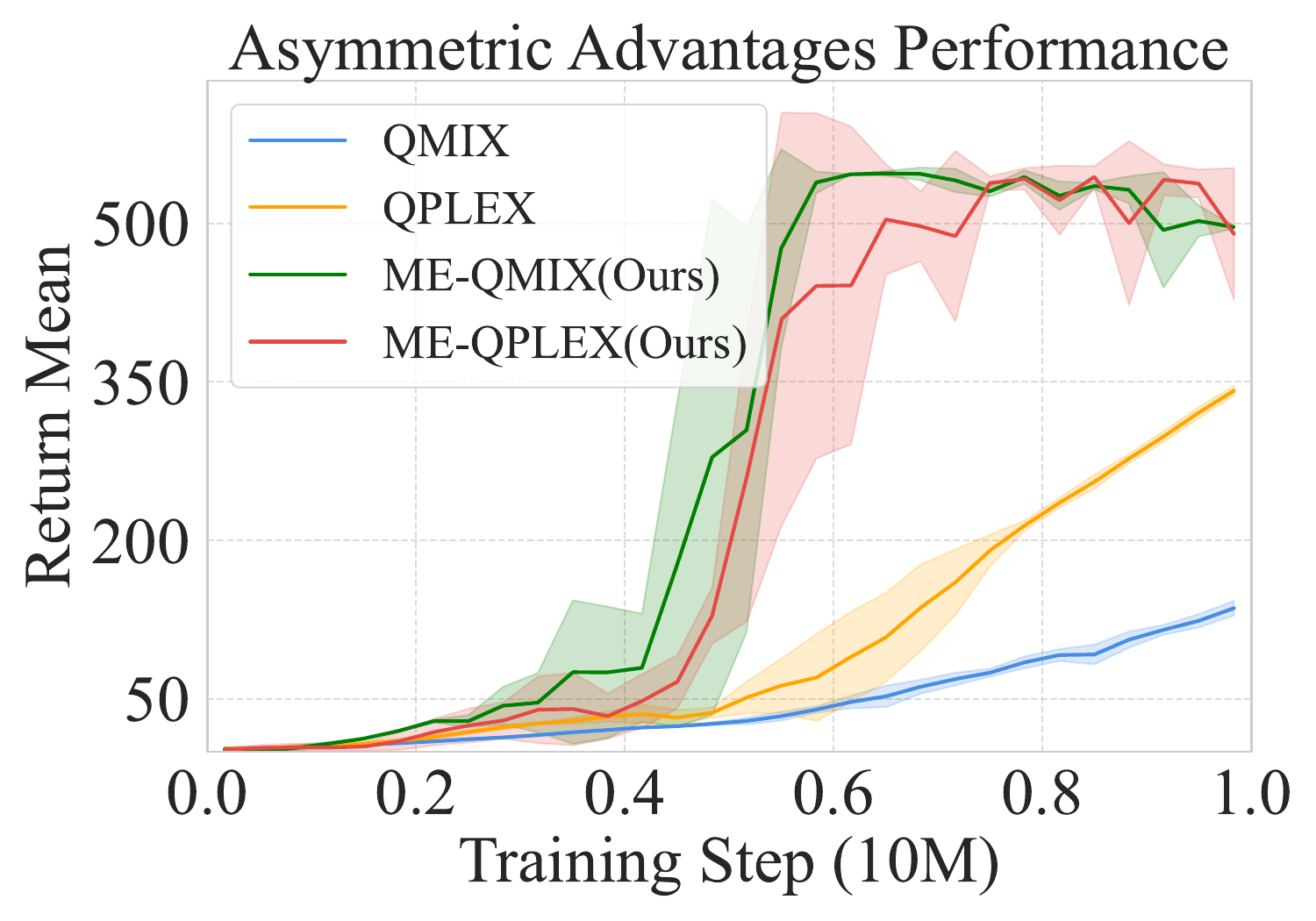}
\caption{The mean return and standard deviation of QMIX, QPLEX, ME-QMIX , ME-QPLEX in Overcooked. It shows that ME-QMIX and ME-QPLEX achieve higher returns and exhibit faster convergence, compared to QMIX and QPLEX which do not adopt maximum entropy.}
\Description{The mean return and standard deviation of QMIX, QPLEX, ME-QMIX , ME-QPLEX in Overcooked.}
\label{fig:cramped_room}
\end{figure}

\begin{table*}[t]
\centering
\caption{
Average win rates of different algorithms on SMAC-v2.
The results we recorded demonstrate that our method, ME-IGM, outperforms all previous baseline algorithms and achieves performance comparable to IL+MARL methods. Note that the results for MAPPO, IPPO, QMIX, QPLEX, IMAX-PPO, and InQ are cited directly from \citet{bui2024mimicking} as mean values; therefore, no standard deviations are reported for these baselines.}
\begin{tabular}{@{}llcccccccc|ccc@{}}
\toprule
\multicolumn{2}{c}{} & \multicolumn{6}{c}{\textbf{MARL}} & \multicolumn{2}{c}{\textbf{ME-IGM (Ours)}} & \multicolumn{2}{c}{\textbf{IL+MARL}} \\ \cmidrule(lr){3-8} \cmidrule(lr){9-10} \cmidrule(lr){11-12}
Task      & Scenario  & MAPPO & IPPO & QMIX & QPLEX & HASAC & FOP & ME-QMIX & ME-QPLEX & IMAX-PPO & InQ \\ \midrule
\multirow{5}{*}{Protoss} 
& 5\_vs\_5   & 58.0 & 54.6 & \underline{70.2} & 53.3 & 22.9 & 37.8 & \textbf{75.9$\pm$3.0} & 69.8$\pm$1.8 & 68.1 & 78.7 \\
& 10\_vs\_10 & 58.3 & 58.0 & 69.0 & 53.7 & 12.1 & 8.6 & \underline{78.5$\pm$2.7} & \textbf{78.7$\pm$3.4} & 59.6 & 79.8 \\
& 10\_vs\_11 & 18.2 & 20.3 & 42.5 & 22.8 & 5.7 & 0.4 & \textbf{50.4$\pm$2.9} & \underline{44.9$\pm$3.2} & 21.3 & 48.7 \\
& 20\_vs\_20 & 38.1 & 44.5 & \underline{69.7} & 27.2 & - & 2.0 & \textbf{73.9$\pm$3.0} & 38.3$\pm$4.0 & 76.3 & 80.6 \\
& 20\_vs\_23 &  5.1 &  4.1 & 16.5 &  4.8 & - & 0.3 & \textbf{23.4$\pm$1.0} & \underline{18.7$\pm$1.3} & 11.8 & 24.2 \\ \midrule

\multirow{5}{*}{Terran} 
& 5\_vs\_5   & 52.0 & 56.2 & 58.4 & 70.0 & 37.9 & 57.1 & \textbf{70.2$\pm$2.6} & \underline{70.0$\pm$4.8} & 53.3 & 69.9 \\
& 10\_vs\_10 & 58.1 & 57.3 & 65.8 & 66.1 & 20.4 & 15.8 & \textbf{72.6$\pm$1.8} & \underline{69.4$\pm$0.2} & 58.4 & 72.2 \\
& 10\_vs\_11 & 28.6 & 31.0 & 39.4 & 41.4 & 9.1 & 7.5 & \underline{49.6$\pm$1.6} & \textbf{51.0$\pm$0.4} & 28.4 & 53.9 \\
& 20\_vs\_20 & 52.8 & 49.6 & \underline{57.6} & 23.9 & - & 0.2 & \textbf{59.3$\pm$1.3} & 45.1$\pm$7.2 & 35.9 & 65.4 \\
& 20\_vs\_23 & 11.2 & 10.0 & 10.0 &  7.0 & - & 0.0 & \textbf{18.7$\pm$0.7} & 17.4$\pm$2.9 &  4.7 & 17.7 \\ \midrule

\multirow{5}{*}{Zerg} 
& 5\_vs\_5   & 41.0 & 37.2 & 37.2 & 47.8 & 29.1 & 35.2 & \underline{50.8$\pm$4.6} & \textbf{52$\pm$3.9} & 48.6 & 55.0 \\
& 10\_vs\_10 & 39.1 & \textbf{49.4} & 40.8 & 41.6 & 17.9 & 8.0 & \underline{49.1$\pm$0.6} & 42.9$\pm$1.8 & 50.6 & 57.6 \\
& 10\_vs\_11 & 31.2 & 26.0 & 28.0 & 31.1 & 14.0 & 1.5 & 32.6$\pm$1.7 & \textbf{45.3$\pm$0.6} & \underline{34.8} & 41.5 \\
& 20\_vs\_20 & 31.9 & 31.2 & 30.4 & 15.8 & - & 0.2 & \underline{34.3$\pm$5.2} & \textbf{42.1$\pm$5.1} & 26.7 & 43.3 \\
& 20\_vs\_23 & \underline{15.8} &  8.3 & 10.1 &  6.7 & - & 0.2 & 14.8$\pm$6.4 & \textbf{17.5$\pm$3.2} &  8.2 & 21.3 \\ 
\bottomrule

\hline
\end{tabular}

\label{tb:smac-v2_results}
\end{table*}

\subsection{Evaluation on Matrix Games}

We first demonstrate the effectiveness of our algorithm in a classic one-step matrix game, shown as Table~\ref{tb:matrix_game1}.  It is a non-monotonic payoff matrix commonly used in previous studies~\cite{son2019qtran,wang2020qplex}. 
We evaluate ME-QMIX, a variant of ME-IGM built upon QMIX, by comparing its performance against two state-of-the-art MARL methods: QMIX~\citep{rashid2020monotonic}, a value decomposition MARL method, and FOP~\citep{zhang2021fop}, a Maximum Entropy MARL method.
We are concerned with whether the joint action with the highest selection probability corresponds to the globally maximal Q value, ensuring optimal returns when executing a deterministic policy.


As shown in Table~\ref{tb:matrix_game1}, QMIX suffers from high estimation error in payoff values of this non-monotonic game, due to its monotonicity constraints (as introduced in Section \ref{rws}). Its deterministic policy converges to a suboptimal joint action (B, B) with a reward of 0. FOP, which employs a network architecture similar to QPLEX, is capable of accurately estimating value functions (i.e., the payoff). However, its policy tends to select suboptimal joint actions more frequently due to the misalignment between local policies and the maximum global Q-values, as discussed in Section \ref{misalignment}. 
This observation highlights that the misalignment between learned policies and optimal joint actions is not confined to IGM-based approaches. 
ME-QMIX exhibits smaller estimation errors for the payoff of optimal joint actions (which is our primary concern) while having larger errors for suboptimal ones. Crucially, the joint action with the maximum Q-value in ME-QMIX aligns with the true optimal joint action, and the local policy assigns the highest probability to selecting this optimal joint action, avoiding the misalignment issue through the use of OPTs.

\subsection{Evaluation on Overcooked}

To demonstrate the advantage of employing the principle of maximum entropy, we compare ME-IGM with two leading value decomposition MARL methods (which do not involve maximum entropy): QMIX and QPLEX, on the Overcooked benchmark~\cite{carroll2019utility} -- a widely used coordination-focused multi-agent environment. In particular, we evaluate them in two representative layouts: cramped room and asymmetric advantages. 
These scenarios pose distinct challenges, including close collaboration under spatial constraints and asymmetric advantages in task division, and require coordinated exploration.
For each scenario, we conduct repeated experiments with different random seeds to get the training curves of 10 million environment steps. As shown in Figure \ref{fig:cramped_room}, ME-QMIX and ME-QPLEX exhibit significantly faster convergence and achieve higher performance levels in the early stages of training compared to their counterparts, QMIX and QPLEX.


\subsection{Evaluation on SMAC-v2}

We further conduct a comprehensive evaluation of ME-IGM on the SMACv2~\cite{ellis2022smacv2} benchmark across 15 scenarios, comparing its performance against 8 baselines. The results, as shown in Table~\ref{tb:smac-v2_results}, are obtained by training each algorithm on each scenario for \(10\) millions of environment steps. Each value in the table represents the average performance over three runs with different random seeds. ME-IGM outperforms the baseline algorithms across all scenarios by a significant margin. MAPPO and IPPO, as on-policy algorithms, have lower sampling efficiency, resulting in lower win rates within \(10\)M environment steps. QMIX and QPLEX, lacking exploration, tend to converge to local optimum. HASAC~\citep{liu2023maximum} and FOP perform poorly across all scenarios. 
Surprisingly, ME-IGM achieves performance on par with (or even surpassing) IL+MARL methods, despite not relying on expert opponent demonstrations or opponent modeling. The results of MAPPO, IPPO, QMIX, QPLEX, IMAX-PPO, InQ are taken from~\cite{bui2024mimicking}. A detailed comparison with actor-critic baselines is presented in the Appendix.


\subsection{Ablation Study}
\label{ap:ablation_study}

\begin{figure}[h!]
\centering
\includegraphics[width=0.35\textwidth]{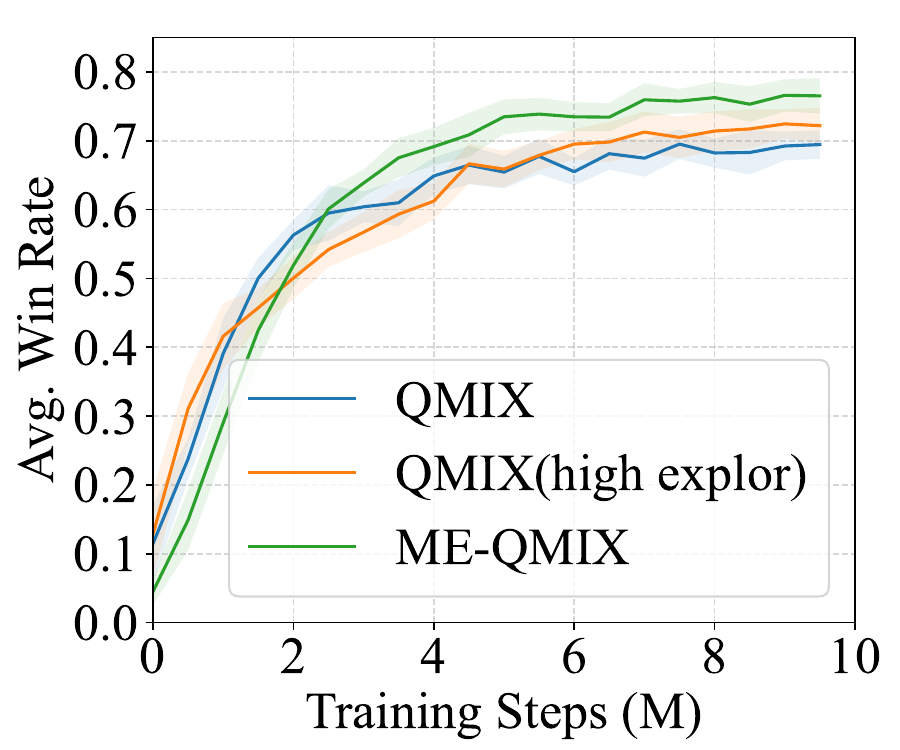}
\caption{Ablation study on exploration strategies in QMIX. Comparison between ME-QMIX and a modified QMIX with extended epsilon annealing on the Protoss 5v5 map. The results show that simply increasing epsilon-greedy exploration is insufficient for improving performance. In contrast, ME-QMIX enables QMIX agents to perform more structured exploration through the maximum entropy framework, leading to higher returns.}
\Description{Comparison of exploration strategies in CTDE MARL}
    \label{fig:qmix_exploration}
\end{figure}

\textbf{Benefits of entropy-driven exploration:} To further demonstrate the benefits of using the maximum entropy framework for exploration, we compare ME-QMIX with a modified version of QMIX, which incorporates significantly higher exploration by applying an extended epsilon annealing period in the epsilon-greedy method, on Protoss 5v5. Figure~\ref{fig:qmix_exploration} illustrates that simply increasing the epsilon annealing time for exploration does not improve QMIX's performance. ME-QMIX achieves superior performance, with the systematic exploration enabled by the maximum entropy framework.

\begin{table}[t]
\caption{Our algorithm (+OPT) vs. its two ablated versions: +Entropy, obtained by removing OPTs, and QMIX, obtained by further removing softmax local policies. Results on three SMAC-v2 scenarios show that removing OPTs (+Entropy) leads to lower win rates due to misalignment between local and global objectives, while incorporating OPTs ensures consistent and monotonic policy improvement.}
\centering
\begin{tabular}{lcccc}
\hline
\textbf{} & \textbf{QMIX} & \textbf{+Entropy} & \textbf{+\text{OPT}} \\
\hline
\textbf{Protoss 5vs5} & 0.68 $\pm$ 0.04 & 0.73 $\pm$ 0.05 & 0.74 $\pm$ 0.04  \\ 
\textbf{Terran 5vs5} & 0.68 $\pm$ 0.05 & 0.67 $\pm$ 0.04 & 0.70 $\pm$ 0.05  \\ 
\textbf{Zerg 5vs5} & 0.41 $\pm$ 0.05 & 0.40 $\pm$ 0.04 & 0.48 $\pm$ 0.05  \\ \hline
\end{tabular}
\label{tab:instruction_tuning}
\end{table}

\noindent\textbf{The necessity of using OPTs}: In Table \ref{tab:instruction_tuning}, we compare QMIX with its two variants: +Entropy and +OPT. Both variants learn local stochastic policies based on Q-functions trained with the QMIX objective. Specifically, +OPT corresponds to our proposed ME-IGM, while +Entropy removes the OPTs from ME-IGM (i.e., no OPTs in Equations (\ref{eq:real_policy_improvement}), (\ref{eq:Lpi}), and (\ref{eq:policyi})). 
We test each algorithm on three SMAC-v2 scenarios across four random seeds, reporting the average win rate along with its standard deviation in the table.
Without OPTs, +Entropy suffers from the misalignment problem discussed in Section \ref{misalignment} and fails to ensure monotonic policy improvement, as analyzed in Section \ref{sec:train_opt}, 
leading to a performance decline in two out of three scenarios. 
On the other hand, +OPT achieves a substantial performance improvement across all evaluated scenarios.

\begin{figure}[t]
    \centering
    \begin{subfigure}{0.35\textwidth}
        \centering
        \includegraphics[width=\textwidth]{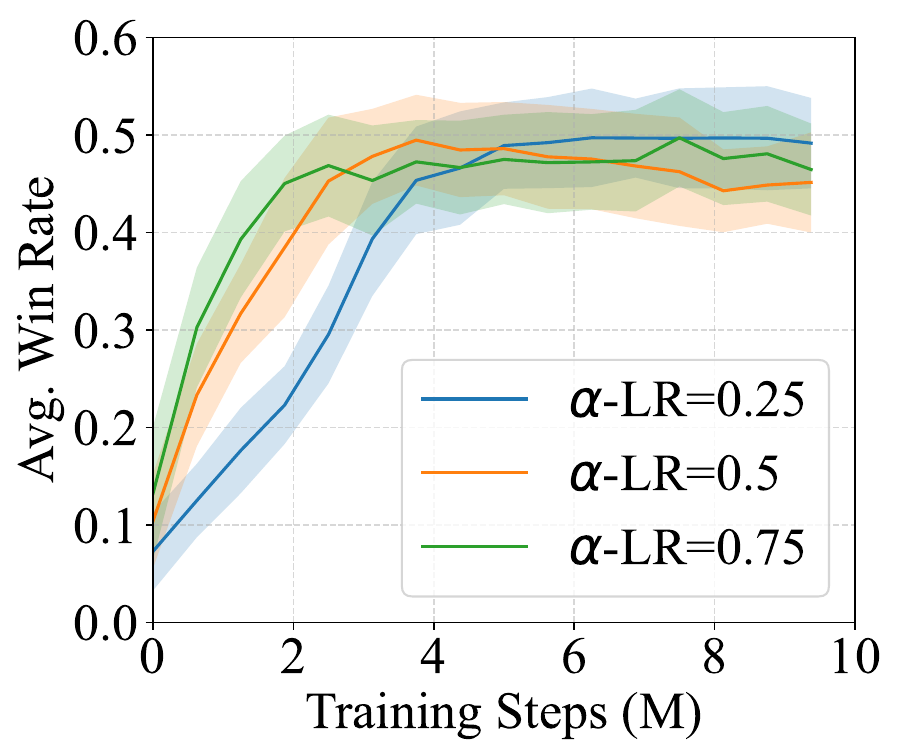}
        \caption{Impact of the learning rate for \( \alpha_\omega \)}
        \label{fig:alpha_LR}
    \end{subfigure}
    \hfill
    \begin{subfigure}{0.35\textwidth}
        \centering
        \includegraphics[width=\textwidth]{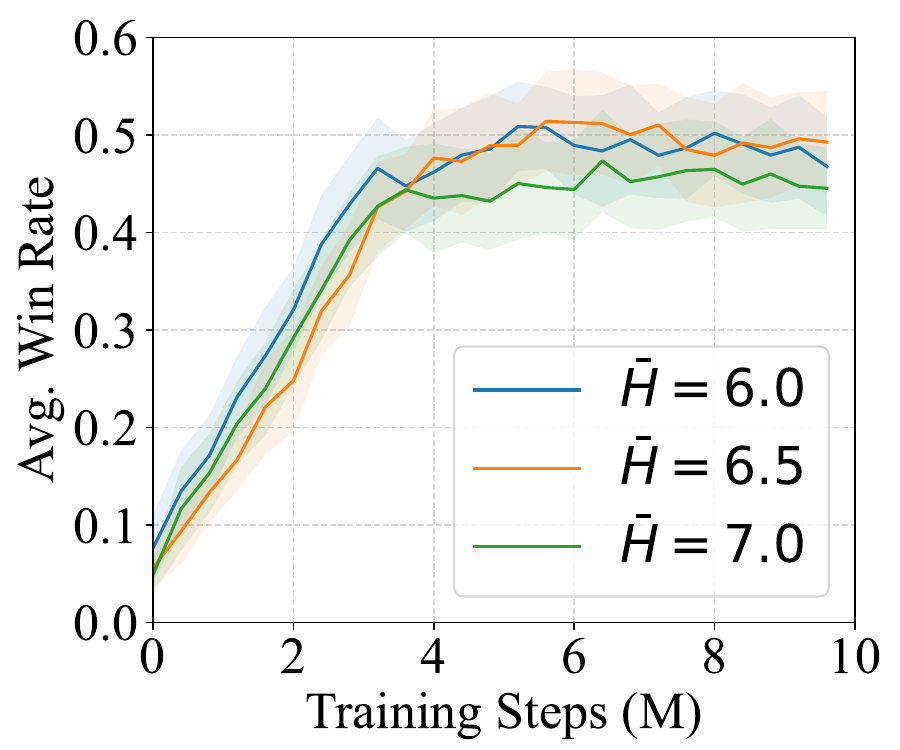}
        \caption{Impact of the target entropy \( \bar{\mathcal{H}} \)}
        \label{fig:alpha_target}
    \end{subfigure}
    \caption{The impact of different hyperparameters for training $\alpha_\omega$ on the performance of ME-QMIX. (a) A lower learning rate for $\alpha_\omega$ corresponds to increased exploration. (b) ME-QMIX's performance is not sensitive to the choice of the target entropy \( \bar{\mathcal{H}} \) in Equation (\ref{eq:policyi}).}
    \Description{The impact of different hyperparameters on the performance of ME-QMIX}
    \label{fig:alpha_hyper-para}
    \vspace{5mm}
\end{figure}

\begin{table}[t]
    \caption{Ablation study on the effect of the Order-Preserving Transformation (OPT). We compare ME-QMIX (OPT) with a variant that replaces OPT layers with a standard 3-layer MLP, which does not preserve the order of input values. Results on three SMAC-v2 scenarios show that ME-QMIX (OPT) consistently achieves higher win rates, highlighting that enforcing order preservation is essential for mitigating local–global misalignment and achieving stable policy improvement.}
    \centering
    \begin{tabular}{lccc}
    \hline
         & \textbf{ME-QMIX (MLP)} & \textbf{ME-QMIX (OPT)} \\
    \hline
    \textbf{Protoss 5v5} & $0.58 \pm 0.03$ & $0.74 \pm 0.04$ \\
    \textbf{Terran 5v5} & $0.61 \pm 0.04$ & $0.70 \pm 0.05$ \\
    \textbf{Zerg 5v5} & $0.41 \pm 0.04$ & $0.48 \pm 0.05$ \\
    \hline
    \end{tabular}
    \label{tab:w_and_wo_opt}
\end{table}

To further validate the design of the order-preserving transformation, we replace the OPT layers in ME-QMIX with a 3-layer MLP, which does not enforce non-negative weights, referring to this variant as ME-QMIX (MLP). As shown in Table \ref{tab:w_and_wo_opt}, our method, ME-QMIX (OPT), consistently outperforms ME-QMIX (MLP) across three different SMAC-v2 scenarios. This underscores the importance of the order-preserving transformation (i.e., Eq. \eqref{eq:order_preserving_transformation}) in resolving the misalignment between local policies and the maximum global Q-value, thereby enhancing overall performance.

\noindent\textbf{The impact of hyperparameters}: Compared to QMIX, ME-QMIX introduces two additional hyperparameters: the learning rate and the target entropy \( \bar{\mathcal{H}} \) for updating \( \alpha_\omega \) in Eq. (\ref{eq:policyi}). We set the other hyperparameters of ME-QMIX following the design choices of QMIX. To analyze the impact of these additional parameters, we conduct a parameter scan on the learning rate of \( \alpha_\omega \) and the target entropy \( \bar{\mathcal{H}} \).  Figure~\ref{fig:alpha_LR} illustrates that the learning rate of \( \alpha_\omega \) influences the exploration level. A lower learning rate results in higher policy entropy at the beginning of training, promoting exploration. While this initially leads to weaker performance (e.g., the blue curve in Figure~\ref{fig:alpha_LR}), it eventually achieves higher returns. Furthermore, the results in Figure~\ref{fig:alpha_target} suggest that the algorithm’s performance is not sensitive to the choice of the target entropy. Notably, in this study, we did not fine-tune configurations for each individual testing scenario. Please check the Appendix for the tables of hyperparameters.

\section{Conclusion}

In this paper, we introduce ME-IGM, a novel maximum entropy MARL approach that applies the CTDE framework and satisfies the IGM condition. We identify and address the misalignment problem between local policies and the maximum global Q-value (in maximum entropy MARL) by introducing an order-preserving transformation (OPT). We further propose a theoretically sound and straightforward objective for updating these transformation operators. Empirically, we demonstrate the effectiveness of ME-IGM on matrix games and its state-of-the-art performance on the challenging Overcooked and SMAC-v2 benchmarks. However, as a value-based method, ME-IGM is currently limited to tasks with discrete action spaces. In future work, we plan to extend its use to tasks involving continuous action spaces.



\balance
\bibliographystyle{ACM-Reference-Format} 
\bibliography{ref}
\clearpage

\appendix
\onecolumn

\clearpage

\section*{Acknowledgment}

This work was supported in part by the U.S. Army Futures Command under Contract No. W519TC-23-C-0030.

\section{TD(\(\lambda\))}
\label{ap:td-lambda}

In this work we use TD(\(\lambda\)) to update the Q-value.
In this paragraph, we will derive the expression of TD(\(\lambda\)) in the context of Maximum Entropy Reinforcement Learning. Specifically, we aim to demonstrate the derivation of the following equation:

\begin{equation}
\label{eq:td_lambda}
\begin{aligned}
\mathcal{T}^{\pi_{jt}}_\lambda Q_{tot}(s_t, \mathbf{u_t}) - Q_{tot}(s_t, \mathbf{u_t})
=
\sum_{l=0}^{\infty} (\gamma \lambda)^l \delta^V_{t+l}
\end{aligned}
\end{equation}
where $\delta_{t}^V = -Q_{tot}(s_t, \mathbf{u}_t) + r_t + \gamma V_{tot}(s_{t+1})$.

\begin{proof}
The one-step TD error can be expressed in the following form:
\begin{equation}
\begin{aligned}
&\delta_{t}^V = -Q_{tot}(s_t, \mathbf{u}_t) + r_t + \gamma V_{tot}(s_{t+1}) \\
&\delta_{t+1}^V = -Q_{tot}(s_{t+1}, \mathbf{u}_{t+1}) + r_{t+1} + \gamma V_{tot}(s_{t+2}) 
\end{aligned}
\end{equation}

Next, the k-step TD error can be represented in the following form:
\begin{equation}
\begin{aligned}
&\epsilon_t^{(1)} := -Q_{tot}(s_t, \mathbf{u}_t) + r_t + \gamma V_{tot}(s_{t+1}) = \delta^V_t\\
&\epsilon_t^{(2)} := -Q_{tot}(s_t, \mathbf{u}_t) + r_t + \gamma r_{t+1} {- \log\pi_{jt}(\mathbf{u_{t+1}}|s_{t+1})}  + \gamma^2 V_{tot}(s_{t+1}) = \delta^V_t + \gamma \delta^V_{t+1} \\
&\epsilon_t^{(k)} = \sum_{l=0}^{k-1} \gamma^l \delta^V_{t+l}
\end{aligned}
\end{equation}

\begin{equation}
\begin{aligned}
&\epsilon_t^{td(\lambda)} := (1 - \lambda) \left( \epsilon_t^{(1)} + \lambda \epsilon_t^{(2)} + \lambda^2 \epsilon_t^{(3)} + \cdots \right) \\
&= (1 - \lambda) \big( \delta_t^V + \lambda (\delta_t^V + \gamma \delta_{t+1}^V) + \lambda^2 (\delta_t^V + \gamma \delta_{t+1}^V + \gamma^2 \delta_{t+2}^V) + \cdots \big) \\
&= (1 - \lambda) \left( \delta_t^V (1 + \lambda + \lambda^2 + \cdots) + \gamma \delta_{t+1}^V (\lambda + \lambda^2 + \cdots) \right. \left. + \gamma^2 \delta_{t+2}^V (\lambda^2 + \lambda^3 + \lambda^4 + \cdots) + \cdots \right) \\
&= (1 - \lambda) \left( \delta_t^V \left( \frac{1}{1 - \lambda} \right) + \gamma \delta_{t+1}^V \left( \frac{\lambda}{1 - \lambda} \right) + \cdots \right) \\
&= \sum_{l=0}^{\infty} (\gamma \lambda)^l \delta_{t+l}^V
\end{aligned}
\label{eq:td_lambda_proof}
\end{equation}

\end{proof}

It is important to note that when computing the TD lambda error in the context of a reverse view and maximum entropy reinforcement learning, it is necessary to subtract the $\log(\pi)$ term from the \(t+1\) step's return as well as subtract $\log(
\pi)$ from $Q_{t+1}$ in order to obtain results consistent with Equation~\ref{eq:td_lambda_proof}. 

\clearpage

\section{Proof of Theorem~\ref{thm:policy_improvement_gap}}
\label{ap:proof_epsilon}

\begin{lemma}[Joint Soft Policy Improvement]
Suppose $\pi_{jt}^{\text{old}}\in\Pi$ and consider $\pi_{jt}^{\text{new}}$ to be the solution to Equation~\ref{eq:kl_divergence}. Then $Q^{\pi_{jt}^{\text{new}}}(s_t, \mathbf{u}_t) \geq Q^{\pi_{jt}^{\text{old}}}(s_t, \mathbf{u}_t)$ for all $(s_t, \mathbf{u}_t) \in S \times \mathbf{U}$ with $|\mathbf{U}| < \infty$.

\begin{proof} 

Let $\pi_{jt}^{\text{old}} \in \Pi$ and let $Q^{\pi_{jt}^{\text{old}}}$ and $V^{\pi_{jt}^{\text{old}}}$ be the corresponding soft state-action value and soft state value. The update rule of $\pi_{\text{new}}$ can be defined as:
\begin{equation}
\begin{aligned}
\pi_{\text{new}}(\cdot \mid s_t) &= \arg\min_{\pi' \in \Pi} \text{D}_{KL} \left( \pi'(\cdot) \parallel \exp(Q^{\pi_{jt}^{\text{old}}}(s_t, \cdot)) - \log Z^{\pi_{jt}^{\text{old}}}(s_t)\right) \\
&= \arg\min_{\pi' \in \Pi} J_{\pi_{jt}^{\text{old}}}\left(\pi'(\cdot \mid s_t)\right),
\end{aligned}
\end{equation}
where $Z^{\pi_{jt}^{\text{old}}}(s_t)$ is the normalization term.

Since we can always choose $\pi_{\text{new}} = \pi_{\text{old}} \in \Pi$, the following inequality always hold true: $J_{\pi_{jt}^{\text{old}}}(\pi_{\text{new}}(\cdot \mid s_t)) \leq J_{\pi_{jt}^{\text{old}}}(\pi_{jt}^{\text{old}}(\cdot \mid s_t))$. Hence

\begin{equation}
\begin{aligned}
&\mathbb{E}_{\mathbf{u}_t \sim \pi_{jt}^{\text{new}}}\left[ \log(\pi_{jt}^{\text{new}}(\mathbf{u}_t \mid s_t)) - Q^{\pi_{jt}^{\text{old}}}(s_t, \mathbf{u}_t) + \log Z^{\pi_{jt}^{\text{old}}}(s_t) \right] \\
&\leq \mathbb{E}_{\pi_{jt}^{\text{old}}}\left[ \log(\pi_{jt}^{\text{old}}(\mathbf{u}_t \mid s_t)) - Q^{\pi_{jt}^{\text{old}}}(s_t, \mathbf{u}_t) + \log Z^{\pi_{jt}^{\text{old}}}(s_t) \right],
\end{aligned}
\end{equation}

and the inequality reduces to the following form since partition function $Z^{\pi_{jt}^{\text{old}}}$ depends only on the state, 

\begin{equation}
\label{eq:pg_ineqaulity}
\mathbb{E}_{\mathbf{u}_t \sim \pi_{jt}^{\text{new}}}\left[ Q^{\pi_{jt}^{\text{old}}}(s_t, \mathbf{u}_t) - \log \pi_{jt}^{\text{new}}(\mathbf{u}_t \mid s_t) \right] \geq V^{\pi_{jt}^{\text{old}}}(s_t).
\end{equation}

Next, consider the soft Bellman equation:

\begin{equation}
\begin{aligned}
Q^{\pi_{jt}^{\text{old}}}(s_t, \mathbf{u}_t) 
&= r(s_t, \mathbf{u}_t) + \gamma \mathbb{E}_{s_{t+1} \sim p}\left[ V^{\pi_{jt}^{\text{old}}}(s_{t+1}) \right] \\
&\leq r(s_t, \mathbf{u}_t) + \gamma \mathbb{E}_{s_{t+1} \sim p}\big[ \mathbb{E}_{\mathbf{u}_{t+1} \sim \pi_{jt}^{\text{new}}}\big[ Q^{\pi_{\text{old}}}(s_{t+1}, \mathbf{u}_{t+1}) - \log \pi_{jt}^{\text{new}}(\mathbf{u}_{t+1} \mid s_{t+1}) \big] \big]\\
\vdots\\
&\leq Q^{\pi_{jt}^{\text{new}}}(s_t, \mathbf{u}_t),
\end{aligned}
\end{equation}
where we can repeatedly expand $Q^{\pi_{jt}^{\text{old}}}$ on the RHS by applying the soft Bellman equation, progressing from the second line to the final one.

\end{proof}
\end{lemma}

\begin{theorem}
Denote the policy before the policy improvement as \(\pi_{\text{old}}\) and the policy achieving the optimality of the improvement step as \(\pi_{\text{new}}\).
Updating the policy with Equation (\ref{eq:real_policy_improvement}) ensures that $Q^{\pi_{\text{old}}}_{tot}(s_t, \mathbf{u_t}) -\epsilon
\leq Q_{tot}^{\pi_{\text{new}}}(s_t, \mathbf{u_t}), \forall s_t,\mathbf{u_t}$.
The gap $\epsilon$ satisfies:
\begin{equation}
\begin{aligned} 
\epsilon \leq \frac{\gamma}{1-\gamma}\mathbb{E}_{s_{t+1}\sim \mathcal{P}}\big[C\sqrt{2 D_{KL}\big(\pi_{\text{old}}(\cdot|s_{t+1})|\pi_{\text{new}}(\cdot|s_{t+1})\big)}\big]
\end{aligned}
\end{equation}
where $\mathcal{P}(s_t, \mathbf{u_t}, \cdot)$ is the transition function and $C  = \underset{\mathbf{u}_{t+1}}{\mathrm{max}}\ \Delta_Q(s_{t+1}, \mathbf{u}_{t+1}) = \underset{\mathbf{u}_{t+1}}{\mathrm{max}}\ |Q_{tot}^{\pi_{\text{old}}}(s_{t+1}, \mathbf{u}_{t+1}) - \sum_i \text{OPT}_i(Q_i^{\pi_{\text{old}}}(o^i_{t+1}, u^i_{t+1}))|$.

\begin{proof}

We define the value function as $V'^{\pi_{jt}^{\text{new}}}(s) = \mathbb{E}_{\mathbf{u}_{t+1} \sim\pi_{jt}^{\text{new}}}\big[ \sum_i\text{OPT}_i(Q_i^{\pi_{jt}^{\text{new}}}(o^i_{t+1}, u^i_{t+1})) - \log \pi_{jt}^{\text{new}}(\mathbf{u}_{t+1} \mid s_{t+1}) \big] \big]$. Then,
similar to Equation~\ref{eq:pg_ineqaulity}, during policy improvement step, we have:

\begin{equation}
\mathbb{E}_{\mathbf{u}_t \sim \pi_{jt}^{\text{new}}}\left[ \sum_i \text{OPT}_i(Q_i^{\pi_{jt}^{\text{old}}}(o^i_t, u^i_t)) - \log \pi_{jt}^{\text{new}}(\mathbf{u}_t \mid s_t) \right] \geq V'^{\pi_{jt}^{\text{old}}}(s_t).
\end{equation}

By expanding the soft Bellman equation, we have:

\begin{equation}
\begin{aligned}
Q_{tot}^{\pi_{jt}^{\text{old}}}(s_t, \mathbf{u}_t) 
&= r(s_t, \mathbf{u}_t) + \gamma \mathbb{E}_{s_{t+1} \sim p}\left[ V_{tot}^{\pi_{jt}^{\text{old}}}(s_{t+1}) \right] \\
&= r(s_t, \mathbf{u}_t) + \gamma\big(\mathbb{E}_{s_{t+1} \sim p}\left[ V'^{\pi_{jt}^{\text{old}}}(s_{t+1}) \right] + \epsilon_1\big) \\
&\leq r(s_t, \mathbf{u}_t) + \gamma \mathbb{E}_{s_{t+1}}\big[ \mathbb{E}_{\mathbf{u}_{t+1} \sim \pi_{jt}^{\text{new}}}\big[ \sum_i \text{OPT}_i(Q_i^{\pi_{jt}^{\text{old}}}(o^i_{t+1}, u^i_{t+1})) - \log \pi_{jt}^{\text{new}}(\mathbf{u}_{t+1} \mid s_{t+1}) \big] \big] + \gamma\epsilon_1\\
&= r(s_t, \mathbf{u}_t) + \gamma \mathbb{E}_{s_{t+1}}\big[ \mathbb{E}_{\mathbf{u}_{t+1} \sim \pi_{jt}^{\text{new}}}\big[ Q_{tot}^{\pi_{\text{old}}}(s_{t+1}, \mathbf{u}_{t+1}) - \log \pi_{jt}^{\text{new}}(\mathbf{u}_{t+1} \mid s_{t+1}) \big] \big] + \gamma(\epsilon_1+\epsilon_2)\\
\vdots\\
&\leq Q_{tot}^{\pi_{jt}^{\text{new}}}(s_t, \mathbf{u}_t) + \epsilon,
\end{aligned}
\end{equation}
where $p$ is the transition function $\mathcal{P}(s_t,a_t,\cdot)$, and $\epsilon$ is the cumulative sum of $\epsilon_1$ and $\epsilon_2$ over time. $\epsilon_1$ and $\epsilon_2$ are defined as follows:
\begin{equation}
\begin{aligned}
\epsilon_1 &= \mathbb{E}_{s_{t+1} \sim p}\left[ V_{tot}^{\pi_{jt}^{\text{old}}}(s_{t+1}) - V'^{\pi_{jt}^{\text{old}}}(s_{t+1}) \right] \\
\epsilon_2 &= \mathbb{E}_{s_{t+1}\sim p}\big[ \mathbb{E}_{\mathbf{u}_{t+1} \sim \pi_{jt}^{\text{new}}}\big[ \sum_i \text{OPT}_i(Q_i^{\pi_{jt}^{\text{old}}}(o^i_{t+1}, u^i_{t+1})) - Q_{tot}^{\pi_{\text{old}}}(s_{t+1}, \mathbf{u}_{t+1})\big]\big]
\end{aligned}
\end{equation}

Then, we have:
\begin{equation}
\begin{aligned}
\epsilon_1 + \epsilon_2 &= \mathbb{E}_{s_{t+1} \sim p}\left[ V_{tot}^{\pi_{jt}^{\text{old}}}(s_{t+1}) - V'^{\pi_{jt}^{\text{old}}}(s_{t+1}) \right] \\
&\quad+ \mathbb{E}_{s_{t+1}\sim p}\big[ \mathbb{E}_{\mathbf{u}_{t+1} \sim \pi_{jt}^{\text{new}}}\big[ \sum_i \text{OPT}_i(Q_i^{\pi_{jt}^{\text{old}}}(o^i_{t+1}, u^i_{t+1})) - Q_{tot}^{\pi_{\text{old}}}(s_{t+1}, \mathbf{u}_{t+1})\big]\big] \\
& = \mathbb{E}_{s_{t+1} \sim p}\left[ \mathbb{E}_{\mathbf{u}_{t+1} \sim \pi_{jt}^{\text{old}}}\big[ Q_{tot}^{\pi_{\text{old}}}(s_{t+1}, \mathbf{u}_{t+1}) - \sum_i \text{OPT}_i(Q_i^{\pi_{jt}^{\text{old}}}(o^i_{t+1}, u^i_{t+1}))\big] \right] \\
&\quad+ \mathbb{E}_{s_{t+1}\sim p}\big[ \mathbb{E}_{\mathbf{u}_{t+1} \sim \pi_{jt}^{\text{new}}}\big[ \sum_i \text{OPT}_i(Q_i^{\pi_{jt}^{\text{old}}}(o^i_{t+1}, u^i_{t+1})) - Q_{tot}^{\pi_{\text{old}}}(s_{t+1}, \mathbf{u}_{t+1})\big]\big] \\
& = \mathbb{E}_{s_{t+1}\sim p}\big[\sum_{\mathbf{u}_{t+1}} \big(\pi_{jt}^{\text{old}}(\mathbf{u}_{t+1}|s_{t+1}) - \pi_{jt}^{\text{new}}(\mathbf{u}_{t+1}|s_{t+1})\big)\big(Q_{tot}^{\pi_{\text{old}}}(s_{t+1}, \mathbf{u}_{t+1}) - \sum_i \text{OPT}_i(Q_i^{\pi_{jt}^{\text{old}}}(o^i_{t+1}, u^i_{t+1}))\big) \big] \\
&\leq \mathbb{E}_{s_{t+1}\sim p}\big[\sum_{\mathbf{u}_{t+1}} |\pi_{jt}^{\text{old}}(\mathbf{u}_{t+1}|s_{t+1}) - \pi_{jt}^{\text{new}}(\mathbf{u}_{t+1}|s_{t+1})||Q_{tot}^{\pi_{\text{old}}}(s_{t+1}, \mathbf{u}_{t+1}) - \sum_i \text{OPT}_i(Q_i^{\pi_{jt}^{\text{old}}}(o^i_{t+1}, u^i_{t+1}))| \big] \\
&\leq \mathbb{E}_{s_{t+1}\sim p}\big[C \cdot \sum_{\mathbf{u}_{t+1}} |\pi_{jt}^{\text{old}}(\mathbf{u}_{t+1}|s_{t+1}) - \pi_{jt}^{\text{new}}(\mathbf{u}_{t+1}|s_{t+1})|\big],
\end{aligned}
\end{equation}
where $C = \underset{\mathbf{u}_{t+1}}{\max}\space\space |Q_{tot}^{\pi_{\text{old}}}(s_{t+1}, \mathbf{u}_{t+1}) - \sum_i \text{OPT}_i(Q_i^{\pi_{jt}^{\text{old}}}(o^i_{t+1}, u^i_{t+1}))|.$ Due to the Pinsker's inequality~\cite{pinsker1964information}, we have:
\begin{equation}
\begin{aligned}
\epsilon \leq \frac{\gamma}{1-\gamma}|\epsilon_1 + \epsilon_2| 
&\leq \frac{\gamma}{1-\gamma}\mathbb{E}_{s_{t+1}\sim p}\big[C \cdot \sum_{\mathbf{u}_{t+1}} |\pi_{jt}^{\text{old}}(\mathbf{u}_{t+1}|s_{t+1}) - \pi_{jt}^{\text{new}}(\mathbf{u}_{t+1}|s_{t+1})|\big] \\
&\leq \frac{\gamma}{1-\gamma}\mathbb{E}_{s_{t+1}\sim p}\big[C\cdot \sqrt{2 D_{KL}\big(\pi_{jt}^{\text{old}}(\cdot|s_{t+1})|\pi_{jt}^{\text{new}}(\cdot|s_{t+1})\big)}\big]
\end{aligned}
\end{equation}

\end{proof}
\end{theorem}

\clearpage

\subsection{Convergence Analysis}
\label{ap:convergence}

In this section, we provide a theoretical analysis of the convergence properties of the proposed ME-IGM algorithm. 
Our analysis leverages the theoretical framework of Regularized Modified Policy Iteration (RMPI) established in previous literature

\subsubsection{Theoretical Foundation: Regularized MPI}

As proved by Geist et al.~\citep{smirnova2019convergenceapproximateregularizedpolicy}, the convergence of policy iteration with time-varying regularization can be established through the following theorem:

\begin{theorem}[RMPI Convergence]
\label{thm:rmpi}
Consider the Regularized Modified Policy Iteration algorithm with time-varying regularization functions $\Omega_t$. Let the sequence $(\lambda_t)_{t}$ be a uniform bound for $\Omega_t$, such that:
\begin{equation} \label{eq:reg_app}
\sup_{\pi} \|\Omega_t(\pi)\|_{\infty} := \sup_{\pi, s} |\Omega_t(\pi(\cdot|s))| \le \lambda_t.
\end{equation}
The gap between the utility of the learned policy after $N$ iterations, $V_{N,\Omega}$, and the optimal utility $V^*$ is bounded by:
\begin{equation} \label{eq:reg-mpi-upper-bound_app}
\|V_{N,\Omega} - V^*\|_{\infty} \le \frac{2}{1-\gamma} \left( \Lambda_N + \gamma^N \|V_{0,\Omega} - V^*\|_{\infty} \right),
\end{equation}
where $\Lambda_N := \left(1 + \frac{1-\gamma^m}{1-\gamma}\right) \sum_{t=1}^{N-1} \gamma^{N-t} \lambda_t$. Notably, if $\lambda_t \to 0$ as $t \to \infty$, the algorithm converges to the optimal value function $V^*$.
\end{theorem}

\subsubsection{Mapping ME-IGM to RMPI}

The discrepancy induced by the Order-Preserving Transformation (OPT) and the decentralized policy update can be analyzed as a time-varying regularization term $\Omega_t(\pi)$ within the Regularized MPI framework (which we denote as $\epsilon$ in our derivation).

As derived in Theorem~\ref{thm:policy_improvement_gap}, the error term $\Omega_t(\pi)$ is bounded by: \begin{equation} \label{eq:me-igm-bound} \lambda_t = \frac{\gamma}{1-\gamma} \mathbb{E} \left[ \Delta_Q \cdot \sqrt{D_{KL}(\pi_{t-1} || \pi_t)} \right]. \end{equation} Here, $\Delta_Q$ represents the alignment gap between the global Q-value and the transformed individual Q-values: $\Delta_Q(s,u) \triangleq |Q_{tot}(s,u) - \sum_i OPT_i(Q_i, s)|$. By substituting Eq. \ref{eq:me-igm-bound} into the expression for $\Lambda_N$ in Eq. \ref{eq:reg-mpi-upper-bound_app}, we can quantify the total error induced by the misalignment $\Delta_Q$.

\subsubsection{Empirical Convergence Analysis}

Our current empirical evidence strongly supports the convergence of ME-IGM: 
\begin{enumerate}
    \item Vanishing $\Delta_Q$: Empirical results show that $\Delta_Q^2$ decreases by three orders of magnitude during the training process, indicating that the OPT effectively aligns the decentralized policies with the global Q-function. 
    \item Stable Policy Updates: The term $D_{KL}(\pi_{t-1} || \pi_t)$ remains non-increasing during training, suggesting that the policy updates become increasingly stable.
\end{enumerate}

Given these observations, we can approximately view $\lambda_t \to 0$ as training progresses. Under this approximation, Theorem~\ref{thm:rmpi} guarantees that the ME-IGM algorithm converges to the optimal value function $V^*$, fulfilling the IGM condition while enjoying the exploration benefits of maximum entropy.

\clearpage

\section{Hyper-network Structure}
\label{ap:hyper-net}
\begin{figure*}[t]
\subfloat[]{\begin{centering}
\includegraphics[width=0.5\linewidth]{figures/mixer.pdf}
\end{centering}
}\subfloat[]{\begin{centering}
\includegraphics[width=0.4\linewidth]{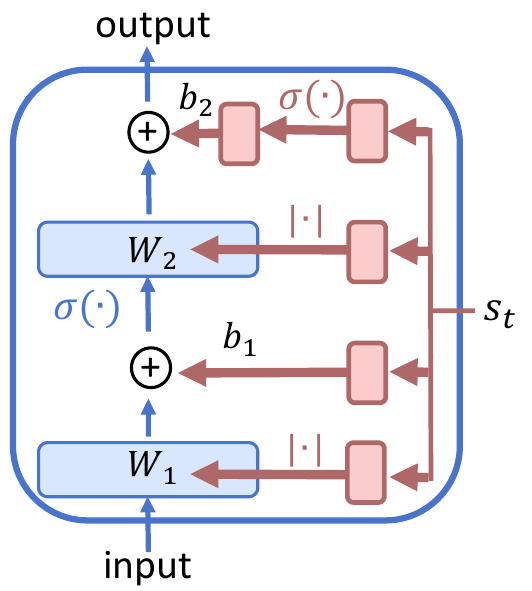}
\end{centering}
}
\caption{(a) The overall network architecture. Single Q-Net takes local observations as input and outputs a distribution over local actions \(u^i\), also denoted as local $Q_i$. Functions \text{OPT}s are order-preserving transformations, ensuring the input and output dimensions are identical and have the same argmax values. Both the input and output of the \text{OPT}s are k-dimensional, where \(k\) is the action dimension. The output of the \text{OPT}, after undergoing a softmax operation, becomes the policy, from which action \(u_t^i\) is sampled. The Mixer network takes local \(Q_i\) as input and outputs the global \(Q_{tot}\). Green blocks utilize a recurrent neural network as the backbone, with all agents sharing parameters. For decentralized execution, only the green network is needed, and the maximum values of their output are selected. (b) The blue blocks' network structure. The red blocks represent the hyper-network, which takes the global state \(s_t\) as input and outputs the network's weights and biases. The weights are constrained to be non-negative. The activation function is denoted by \(\sigma(\cdot)\). The mixer network takes an n-dimensional Q function as input and outputs a one-dimensional \(Q_{tot}\). }
\label{fig:all_net_structure}
\end{figure*}

Figure~\ref{fig:all_net_structure} shows the overall network structure. Specifically, \text{OPT}s share the same network structure as the mixer network, but while the mixer network's input dimension is ($\text{batch\_size, agent\_num}$) and its output dimension is ($\text{batch\_size}$, 1), both the input and output dimensions of \text{OPT}s are ($\text{batch\_size, action\_num}$). The network hyperparameters are as shown in the table below. In the SMAC-v2 experiment, we select a single layer for the \text{OPT} model to ensure fairness in comparison. Using two layers of \text{OPT} would significantly increase the number of parameters, making it less comparable to the baseline algorithms. While this may slightly reduce our method's performance, we adopt this choice for a fair comparison.

\begin{table}[H]
\centering
\begin{tabular}{lccc}
\hline
\textbf{layer name} & \textbf{hyper-net\_embed} & \textbf{embed\_dim} & \textbf{num\_layer} \\
\hline
mixer & 64 & 32 & 2 \\
\text{OPT} & 64 & - & 1 \\
\hline
\end{tabular}
\caption{Hyperparameters used for hyper-networks.}
\end{table}

\clearpage

\section{Comparison with Baselines}
\label{ap:baselines}

The following paragraphs present the differences between our method and those based on the Actor-Critic framework:

1. Using existing maximum entropy MARL algorithms may result in the loss of action order learned through credit assignment. Achieving the IGM condition, which outlines the order of local value functions, is crucial in credit assignment. Our work presents a value-based algorithm that introduces order-preserving transformations, ensuring that the order of the local actor matches the order of the local value functions. However, the commonly used maximum entropy MARL methods, employing the actor-critic framework, may not guarantee alignment between the actor's action order and that of the local Q-function due to approximation errors.

2. Actor-critic methods utilize the KL divergence as the loss function, while our approach uses the MSE loss. In the CTDE framework, all methods maintain a more accurate centralized Q-function alongside multiple imprecise local policies. The core idea is to distill knowledge from the centralized Q-function to the local policies. Actor-critic methods typically minimize the KL divergence between the actor and the softmax critic, whereas our method minimizes the MSE loss between their logits.~\cite{kim2021comparing} pointed out that, in distillation contexts, MSE is a superior loss function compared to KL divergence.

3. AC methods can use global information when training the critic, resulting in higher model capacity. Our method introduces global information through hyper-networks and the mixer network, which results in lower model capacity.

4. In fact, it's challenging to theoretically prove that value-based methods are inherently superior to AC methods. However, through experiments, we have demonstrated that our method surpasses both FOP~\cite{zhang2021fop} and HASAC~\cite{liu2023maximum}.

Below are comparisons with individual papers.

mSAC~\cite{pu2021decomposed} introduces a method similar to MASAC. It employs the AC architecture, thus encountering issues 1 and 2 mentioned above, but does not fully utilize global information to train the critic, hence lacking the advantage noted in point 3.

FOP~\cite{zhang2021fop} shares a similar issue with~\cite{pu2021decomposed}, in that the AC framework it utilizes can actually train the critic with global state information. Furthermore, FOP employs a credit assignment mechanism similar to that of QPLEX~\cite{wang2020qplex}. \cite{hu2021rethinking} experimentally shown that QMIX outperforms QPLEX. Additionally, FOP's proof concept is that under the IGO conditions, if we locally move local policies towards local optimal policies, the distance between the joint policy and the optimal joint policy also decreases. In contrast, our proof approach first establishes that our method is equivalent to the single-agent SAC algorithm and then completes the proof using existing conclusions.

HASAC~\cite{liu2023maximum}, in contrast to FOP~\cite{zhang2021fop}, avoids the restrictive Individual-Global-Optimal (IGO) assumption used for joint policy factorization by employing a theoretically grounded sequential optimization of individual policies via joint soft policy decomposition. However, our primary contribution addresses the misalignment problem, which is orthogonal to HASAC's focus. Moreover, since HASAC does not rely on either the IGM or IGO assumptions, establishing a direct theoretical comparison is nontrivial. Nevertheless, we empirically demonstrate that our method consistently outperforms HASAC across benchmark tasks presented in the paper.

PAC~\cite{zhou2022pac} utilizes global information during centralized training, possessing the advantages noted in point 3. However, it still encounter issues 1 and 2 mentioned above.


\clearpage

\clearpage

\section{Hyper-parameters}
\label{ap:hyper-para}

For all baseline algorithms, we implemented them using the corresponding open-source frameworks and chose the hyperparameters provided by these frameworks for replication. We used 
FOP by ~\citeauthor{zhang2021fop} and HASAC by~\citeauthor{zhong2023heterogeneousagent}. Our method, ME-IGM is built on~\citeauthor{hu2021rethinking}.  The experiments were conducted on a computer equipped with 92 GB of RAM, a 40-core CPU, and a GeForce RTX 2080 Ti GPU. The following are the hyperparameters used in the experiments.


\begin{table}[h]
\centering
\begin{tabular}{@{}lc@{}}
\toprule
Parameter Name & Value \\ \midrule
$\text{n\_rollout\_threads}$ & 8 \\
$\text{num\_env\_steps}$ & 10000000 \\
$\text{warmup\_steps}$ & 10000 \\
$\text{train\_interval}$ & 50 \\
$\text{update\_per\_train}$ & 1 \\
$\text{use\_valuenorm}$ & False \\
$\text{use\_linear\_lr\_decay}$ & False \\
$\text{use\_proper\_time\_limits}$ & True \\
$\text{hidden\_sizes}$ & [256, 256] \\
$\text{activation\_func}$ & $\text{relu}$ \\ 
$\text{use\_feature\_normalization}$ & True \\
$\text{final\_activation\_func}$ & tanh \\
$\text{initialization\_method}$ & orthogonal \\
$\text{gain}$ &  0.01 \\
$\text{lr}$ & 0.0003 \\
$\text{critic\_lr}$ & 0.0005 \\
$\text{auto\_}\alpha$ & True \\
$\alpha\text{\_lr}$ & 0.0003 \\
$\gamma$ & 0.99 \\
$\text{buffer\_size}$ & 1000000 \\
$\text{batch\_size}$ & 1000 \\
$\text{polyak}$ & 0.005 \\
$\text{n\_step}$ & 20 \\
$\text{use\_huber\_loss}$ & False \\
$\text{use\_policy\_active\_masks}$ & True \\
$\text{share\_param}$ & False \\
$\text{fixed\_order}$ & False \\
\bottomrule
\end{tabular}
\caption{HASAC hyperparameters used for SMACv2. We use the hyperparameters for SMAC as specified in the original paper~\cite{liu2023maximum}.}
\end{table}

\begin{table}[h]
\centering
\begin{tabular}{@{}lc@{}}
\toprule
Parameter Name & Value \\ \midrule
Runner & parallel \\
Batch Size Run & 4 \\
Buffer Size & 5000 \\
Batch Size & 128 \\
Optimizer & Adam \\
\(t_{\max}\) & 1005000 \\
Target Update Interval & 200 \\
Mac & \(\text{n\_mac}\) \\
Agent & \(\text{n\_rnn}\) \\
Agent Output Type & q \\
Learner & \(\text{nq\_learner}\) \\
Mixer & qmix \\
Mixing Embed Dimension & 32 \\
Hyper-net Embed Dimension & 64 \\
Learning Rate & 0.001 \\
\(\lambda\) & 0.4 \\ 
$\bar H$ (zerg, protoss) & 0.24 $\times$ num\_ally \\ 
$\bar H$ (terran) & 0.32 $\times$ num\_ally \\ 
$\alpha$ Learning Rate & 0.3 \\ \bottomrule
\end{tabular}
\caption{ME-QMIX hyperparameters used for SMACv2. We utilize the hyperparameters used in SMACv2~\cite{ellis2022smacv2}.}
\end{table}

\begin{table}[h]
\centering
\begin{tabular}{@{}lc@{}}
\toprule
Parameter Name & Value \\ \midrule
Runner & parallel \\
Batch Size Run & 8 \\
Buffer Size & 5000 \\
Batch Size & 128 \\
Optimizer & Adam \\
\(t_{\max}\) & 1005000 \\
Target Update Interval & 200 \\
Mac & \(\text{basic\_mac}\) \\
Agent & \(\text{rnn}\) \\
Agent Output Type & q \\
Learner & \(\text{dmaq\_qatten\_learner}\) \\
Mixer & dmaq \\
Mixing Embed Dimension & 32 \\
Hyper-net Embed Dimension & 64 \\
Learning Rate & 0.001 \\
\(\lambda\) & 0.6 \\ 
$\bar H$ (zerg, protoss) & 0.24 $\times$ num\_ally \\ 
$\bar H$ (terran) & 0.32 $\times$ num\_ally \\ 
$\alpha$ Learning Rate & 0.3 \\ \bottomrule
\end{tabular}
\caption{ME-QPLEX hyperparameters used for SMACv2. We utilize the hyperparameters used in SMACv2~\cite{ellis2022smacv2}.}
\end{table}

\begin{table}[h]
\centering
\begin{tabular}{@{}lc@{}}
\toprule
Parameter Name & Value \\ \midrule
Action Selector & Multinomial \\
\(\epsilon\)-Start & 1.0 \\
\(\epsilon\)-Finish & 0.05 \\
\(\epsilon\)-Anneal Time & 50000 \\
Buffer Size & 5000 \\
\(t_{\max}\) & 1005000 \\
Target Update Interval & 200 \\
Agent Output Type & \(\text{pi\_logits}\) \\
Learner & \(\text{fop\_learner}\) \\
Head Number & 4 \\
Mixing Embed Dimension & 32 \\
Learning Rate & 0.0005 \\
Burn In Period & 100 \\
\(\lambda\) & 0.4 \\ \bottomrule
\end{tabular}
\caption{FOP hyper-parameters used for SMACv2. We use the hyperparameters for SMAC as specified in the original paper~\cite{zhang2021fop}.}
\end{table}

\begin{table}[h]
\centering
\begin{tabular}{@{}lc@{}}
\toprule
Parameter Name & Value \\ \midrule
Runner & parallel \\
Batch Size Run & 4 \\
Buffer Size & 5000 \\
Batch Size & 512 \\
Optimizer & Adam \\
\(t_{\max}\) & 1005000 \\
Target Update Interval & 200 \\
Mac & \(\text{n\_mac}\) \\
Agent & \(\text{n\_rnn}\) \\
Agent Output Type & q \\
Learner & \(\text{nq\_learner}\) \\
Mixer & qmix \\
Mixing Embed Dimension & 32 \\
Hyper-net Embed Dimension & 64 \\
Learning Rate & 0.001 \\
\(\lambda\) & 0.6 \\ 
$\alpha$ Learning Rate & 0.3 \\ \bottomrule
\end{tabular}
\caption{ME-QMIX hyperparameters used for Overcooked. We utilize the hyperparameters used in Overcooked~\cite{carroll2019utility}.}
\end{table}

\begin{table}[h]
\centering
\begin{tabular}{@{}lc@{}}
\toprule
Parameter Name & Value \\ \midrule
Runner & parallel \\
Batch Size Run & 8 \\
Buffer Size & 5000 \\
Batch Size & 512 \\
Optimizer & Adam \\
\(t_{\max}\) & 1005000 \\
Target Update Interval & 200 \\
Mac & \(\text{basic\_mac}\) \\
Agent & \(\text{rnn}\) \\
Agent Output Type & q \\
Learner & \(\text{dmaq\_qatten\_learner}\) \\
Mixer & dmaq \\
Mixing Embed Dimension & 32 \\
Hyper-net Embed Dimension & 64 \\
Learning Rate & 0.001 \\
\(\lambda\) & 0.6 \\ 
$\alpha$ Learning Rate & 0.3 \\ \bottomrule
\end{tabular}
\caption{ME-QPLEX hyperparameters used for Overcooked. We utilize the hyperparameters used in Overcooked~\cite{carroll2019utility}.}
\end{table}

\clearpage

\end{document}